\documentclass[11pt]{article}

\usepackage[margin=1in]{geometry}
\usepackage{microtype}

\usepackage{amsmath}
\usepackage{amssymb}
\usepackage{amsthm}
\usepackage{mathtools}

\usepackage{graphicx}
\graphicspath{{figures/}}
\usepackage{booktabs}
\usepackage{array}
\usepackage{tabularx}
\usepackage{siunitx}
\usepackage[font=small,labelfont=bf]{caption}

\usepackage{algorithm}
\usepackage{algpseudocode}
\usepackage[titletoc,title]{appendix}
\usepackage[section]{placeins}

\usepackage{hyperref}
\usepackage[nameinlink,capitalise,noabbrev]{cleveref}

\usepackage[numbers,sort&compress]{natbib}

\makeatletter
\renewcommand{\theHALG@line}{\thealgorithm.\arabic{ALG@line}}
\makeatother

\newcommand{\PP}{\mathbb{P}}
\newcommand{\ind}{\mathbf{1}}
\newcommand{\countfrac}[2]{\makebox[1.35em][r]{#1}/\makebox[1.35em][l]{#2}}
\newcommand{\winloss}[2]{\makebox[1.35em][r]{#1}--\makebox[1.0em][l]{#2}}
\newcommand{\tablehead}[1]{\multicolumn{1}{c}{#1}}
\newcolumntype{L}[1]{>{\raggedright\arraybackslash}p{#1}}
\newcolumntype{C}[1]{>{\centering\arraybackslash}p{#1}}
\newcolumntype{Y}{>{\raggedright\arraybackslash}X}
\theoremstyle{plain}
\newtheorem{theorem}{Theorem}
\newtheorem{lemma}{Lemma}
\newtheorem{corollary}{Corollary}
\AddToHook{env/theorem/begin}{\crefalias{section}{theorem}}
\AddToHook{env/lemma/begin}{\crefalias{section}{lemma}}
\AddToHook{env/corollary/begin}{\crefalias{section}{corollary}}

\theoremstyle{definition}

\AddToHook{env/definition/begin}{\crefalias{section}{definition}}

\theoremstyle{remark}

\AddToHook{env/remark/begin}{\crefalias{section}{remark}}

\hypersetup{
  hidelinks,
  pdftitle={Conditional Inference Trees and Forests for Feature Selection},
  pdfauthor={Robert Milletich, Justin Downes, Steve Goley, Newel Hirst}
}

\title{Conditional Inference Trees and Forests for Feature Selection}
\author{%
  Robert Milletich, Justin Downes, Steve Goley, Newel Hirst\\
  Amazon Web Services\\
  \texttt{\char`\{rmilleti,jusdow,sgoley,nhirst\char`\}@amazon.com}
}
\date{}

\begin{document}
\maketitle

\begin{abstract}
Conditional inference trees (CIT) and conditional inference forests (CIF)
reduce split-selection bias by testing features before choosing split
thresholds, but repeated permutation tests and threshold searches can make these
methods computationally expensive. We study CIT and CIF as top-$k$
feature-ranking methods for downstream prediction using real-data benchmarks,
runtime ablations, and synthetic feature-recovery experiments.
At a fixed node, if the features and permutation budget do not depend on the
node responses, Bonferroni-corrected $+1$ Monte Carlo permutation $p$-values
control nodewise rejection under the complete permutation null. CIF ranks 4th
among 17 classification methods on 22 datasets and 3rd among 18 regression
methods on 8 datasets.
With Bonferroni correction held fixed, the CIF runtime ablations indicate that
adaptive stopping and the number of thresholds searched have the largest measured
effect on runtime: turning off adaptive stopping and using exact threshold
search increase fitting time by 4.0--8.4$\times$ and 1.9--10.8$\times$,
respectively, while downstream score changes are at most 0.011. Sparse high-$p$
simulations indicate that forest feature sampling can leave informative features
out of many split decisions. Overall,
the results support CIF as a top-$k$ feature-ranking method in the
evaluated downstream prediction benchmarks.

\end{abstract}

\section{Introduction}
\label{sec:introduction}

Conditional inference separates feature selection from threshold optimization to
reduce split selection bias. Classical
CART trees choose both a feature and a threshold from the same
label-dependent search, which helps explain the well-documented preference for
features with many admissible cutpoints
\citep{Breiman1984ClassificationAndRegressionTrees,LohShih1997SplitSelectionMethods,KimLoh2001ClassificationTreesUnbiasedMultiwaySplits,Loh2002RegressionTreesUnbiasedVariableSelectionInteractionDetection,Shih2004NoteOnSplitSelectionBiasClassificationTrees,StroblBoulesteixAugustin2007UnbiasedSplitSelectionGini,Strobl2007BiasRFVariableImportance}.
In the conditional inference framework \citep{Hothorn2006UnbiasedRecursivePartitioning,HothornHornikVandeWielZeileis2006LegoSystemConditionalInference},
Stage~A tests features for association with the response at the
current node and selects among them; Stage~B then optimizes a threshold within
the selected feature. That separation addresses split selection bias, but the
classical approach can require substantial computation through large
permutation budgets and exhaustive threshold scans at every node.
The same conditional inference framework also motivated later work on
conditional variable importance \citep{Strobl2008ConditionalVIM}.
A related line of work uses forests directly for variable screening and ranking,
rather than only for final prediction
\citep{GenuerPoggiTuleauMalot2010VariableSelectionUsingRandomForests}.

This paper evaluates conditional inference trees (CIT) and conditional
inference forests (CIF) as feature selection methods. For each seed and
cross-validation fold, a method produces a ranking from the training fold.
Fixed downstream learners are then fit on top-$k$ prefixes of that ranking and
evaluated on held-out data, with results summarized across datasets. The
benchmark asks whether tree- and forest-derived rankings put predictive
variables early enough to improve fixed downstream models.
We also ask whether CIT/CIF runtime hyperparameters reduce computation while
preserving ranking quality, and where CIF recovery is weaker in high-$p$ or
synthetic settings. This feature-selection benchmark is separate from work on
trees and forests as direct
predictors or causal estimators
\citep{Breiman2001RandomForests,WagerAthey2018CausalForests,AtheyTibshiraniWager2019GeneralizedRandomForests}.
Because several methods have tunable configurations, the benchmark chooses one
configuration per method and task. As a sensitivity check, leave-one-dataset-out
(LODO) analyses repeat the configuration choice after holding out each dataset, using
the complete-case panels defined in \cref{sec:experiments-analysis-plan}.

The theorem studies the Stage~A feature test at a fixed node. It assumes that
the features tested at the node and the resample budget are fixed independently
of the node responses, so exhaustive fixed-$B$ $p$-values have a clear
permutation reference distribution under the nodewise complete permutation null.
The benchmarks then evaluate the tree and forest feature rankings empirically.

\paragraph{Contributions.}
This paper contributes a real-data benchmark of conditional-inference feature
rankings for top-$k$ downstream prediction. The benchmark compares CIF with
ctree, cforest, CIT, and other baselines in classification and companion
regression analyses, and includes ablations of the main CIT/CIF runtime
hyperparameters. The paper also states the fixed-node, exhaustive fixed-$B$
Stage~A guarantee
(\cref{thm:plusone-superuniform,cor:stageA-global-null}); the full tree and
forest rankings are evaluated in the benchmarks. Finally, the
high-$p$ and synthetic analyses show cases with weaker top-$k$ recovery and how
forest feature sampling can reduce how often forest splits use informative
features (\cref{sec:boundary}).

\section{Tree-growing method}
\label{sec:method}

At each internal node, conditional inference trees make two decisions. Stage~A
tests features for association with the response and chooses a feature
to split. Stage~B then searches thresholds only for that selected feature; if no
valid split is accepted, the node becomes terminal. In this paper, CIT
denotes the resulting single-tree feature selection method. CIF denotes a
bootstrap ensemble of such trees with feature sampling at each node
and feature rankings formed by aggregating split importances across trees.
CIF-all disables feature subsampling: every nonconstant available
feature can enter Stage~A at each node. This changes the Stage~A feature set,
the Stage~A multiplicity correction, and the per-node computational cost.

\subsection{Stage~A and Stage~B}
\label{sec:stageA-stageB}

At a node $t$, write $(X_t,Y_t)$ for the samples reaching that node. Let
$F_{t,\mathrm{avail}}$ be the feature set carried into the node by the
tree-growing algorithm, including any feature muting inherited from ancestors.
Let $F_{t,\mathrm{nonconst}} \subseteq F_{t,\mathrm{avail}}$ be the features
that are nonconstant on the node samples. After any feature
subsampling, Stage~A considers a set
$F_t \subseteq F_{t,\mathrm{nonconst}}$; its size $m_t=|F_t|$ is the
Bonferroni denominator for that node.
Any label-independent auxiliary randomness used by a selector is denoted by $U$
and is held fixed in the fixed-node guarantee.

When adaptive stopping is enabled during tree growing, Stage~A can evaluate
features sequentially and stop before every feature in $F_t$ has been tested.
Let $E_t \subseteq F_t$ denote the evaluated subset; in the exhaustive fixed-$B$
reference rule, $E_t=F_t$. The fixed-node theorem in \cref{sec:theory} applies
to that reference rule: the features in $F_t$ and the budget $B$ are fixed
independently of the node responses, and every $p_{t,j}$ for $j\in F_t$ uses
that budget. Adaptive tree growing adds response-dependent node selection,
inherited feature muting, and stopping-time scores used for tree construction.

\paragraph{Stage~A (feature selection).}
In the exhaustive fixed-$B$ setting, Stage~A follows the feature-selection and
threshold-search separation of conditional inference trees
\citep{Hothorn2006UnbiasedRecursivePartitioning}. The original conditional
inference tests use the permutation-statistic framework of
\citet{StrasserWeber1999PermutationStatistics}; here, Monte Carlo label
permutations calibrate the configured selector statistic. Stage~A
computes a selector statistic $T^{\mathrm{sel}}_j(X_{t,j},Y_t;U)$ and
permutation $p$-value $p_{t,j}$ for every feature $j \in F_t$, then selects
\begin{equation}
  p_t^{\star,\mathrm{ref}} := \min_{j \in F_t} p_{t,j},
  \qquad
  M_t^\star := \{j \in F_t : p_{t,j}=p_t^{\star,\mathrm{ref}}\}.
\end{equation}
Let $\tau_t=\alpha_{\mathrm{sel}}/m_t$ when Bonferroni is enabled and
$\tau_t=\alpha_{\mathrm{sel}}$ otherwise. If
$p_t^{\star,\mathrm{ref}} \ge \tau_t$, Stage~A returns no feature and the node
becomes terminal. If $p_t^{\star,\mathrm{ref}} < \tau_t$, the selected feature
$j_t^{\star,\mathrm{ref}}$ is drawn uniformly from $M_t^\star$ using reservoir
sampling, and the node enters Stage~B.

During adaptive tree growing, Stage~A uses the same decision structure but can
order features by a preliminary score and evaluate only the subset $E_t$. The
returned value $q_t^{\mathrm{feat}}$ is used to grow the tree; it is not an
exhaustive fixed-$B$ $p$-value covered by \cref{thm:plusone-superuniform}. Under
the exhaustive reference rule with Bonferroni correction,
$\min(1,m_t p_t^{\star,\mathrm{ref}})$ is a conservative complete-null
$p$-value for the fixed-node Stage~A test. When feature muting is enabled,
features with non-rejecting Stage~A $p$-values can be removed from descendant
feature pools. When muting is disabled, descendants keep
$F_{t,\mathrm{avail}}$ except for deterministic constant feature pruning.

\paragraph{Stage~B (threshold search).}
Given a selected feature $j_t^\star$, Stage~B starts from midpoints between
ordered unique values of $X_{t,j_t^\star}$. The exact method uses all such
midpoints. The random method samples a bounded subset of midpoints, while the
percentile and histogram methods choose bounded representative thresholds
derived from the midpoint distribution. Thresholds that would leave either child
below the minimum leaf size are filtered out. If no threshold remains, Stage~B
returns no split.

For the remaining finite threshold set $C_{t,j_t^\star}$, Stage~B selects the
threshold used by the tree. The tested statistic at threshold $c$ is the
weighted child impurity
\[
  \frac{|Y_t^L(c)|}{|Y_t|}\mathrm{Imp}(Y_t^L(c))
  + \frac{|Y_t^R(c)|}{|Y_t|}\mathrm{Imp}(Y_t^R(c)),
\]
where $\mathrm{Imp}$ is the chosen node impurity measure. This statistic is used
in a left-tail test, and the exhaustive fixed-$B$ rule chooses the threshold with
the smallest Stage~B $p$-value, with uniform tie breaking. During adaptive tree
growing, threshold scanning may reorder thresholds and stop after the first
rejecting threshold, so the returned Stage~B score is a tree-construction score
rather than a global minimum over all retained thresholds. Let
$\tau_t^{\mathrm{split}}=\alpha_{\mathrm{split}}/|C_{t,j_t^\star}|$ when
Bonferroni is enabled for Stage~B and
$\tau_t^{\mathrm{split}}=\alpha_{\mathrm{split}}$ otherwise. If no evaluated
Stage~B $p$-value is below $\tau_t^{\mathrm{split}}$, the node becomes terminal.
After a threshold passes the Stage~B test, the final minimum required-improvement
check uses the corresponding weighted impurity decrease. Stage~B chooses the
split used by the tree; the returned Stage~B score is not a valid post-selection
$p$-value. The training skeleton in \cref{alg:citree} summarizes this flow and,
for the benchmarked CIT and CIF configurations, stores two values at each split:
the Stage~A permutation $p$-value for the selected feature and the Stage~B
permutation $p$-value for the selected threshold.

\subsection{Monte Carlo permutation \texorpdfstring{$p$}{p}-values}
\label{sec:plusone}

Let $T_{\mathrm{obs}}$ be the observed statistic and $T_1,\dots,T_B$ the same statistic
computed on $B$ random label permutations. For a right tail test we use the
Phipson--Smyth $+1$ $p$-value
\citep{NorthCurtisSham2002EmpiricalPValuesMonteCarlo,PhipsonSmyth2010PermutationPValues},
where $\ind\{\cdot\}$ denotes the indicator function:
\begin{equation}
  p_{\mathrm{MC}}^{\ge} := \frac{1 + \sum_{b=1}^B \ind\{T_b \ge T_{\mathrm{obs}}\}}{B+1}.
\end{equation}
For a left tail test we analogously define
\begin{equation}
  p_{\mathrm{MC}}^{\le} := \frac{1 + \sum_{b=1}^B \ind\{T_b \le T_{\mathrm{obs}}\}}{B+1}.
\end{equation}
Stage~A uses the right tail convention and Stage~B the left tail convention;
ties count as exceedances in both cases.

\begin{algorithm}[H]
  \caption{Conditional inference tree training skeleton (Stage~A $\rightarrow$ Stage~B)}
  \label{alg:citree}
  \begin{algorithmic}[1]
    \Require Training data $(X,Y)$, hyperparameters $\theta$
    \Ensure Trained tree $\mathcal{T}$
    \State $p \gets$ number of columns of $X$
    \State \Return \Call{GrowNode}{$X,Y,1,\{1,\dots,p\},\theta$}
    \Function{GrowNode}{$X_t,Y_t,d,F_{\mathrm{avail}},\theta$}
      \If{\Call{Stop}{$X_t,Y_t,d,\theta$}}
        \State \Return \Call{Leaf}{$Y_t$}
      \EndIf
      \State $(j_t^\star, q_t^{\mathrm{feat}}, m_t, F'_{\mathrm{avail}}) \gets \Call{StageA}{X_t,Y_t,F_{\mathrm{avail}},\theta}$
      \If{$j_t^\star = \bot$}
        \State \Return \Call{Leaf}{$Y_t$}
      \EndIf
      \State $(c_t^\star, q_t^{\mathrm{split}}) \gets \Call{StageB}{X_t,Y_t,j_t^\star,\theta}$
      \If{$c_t^\star = \bot$}
        \State \Return \Call{Leaf}{$Y_t$}
      \EndIf
      \State $(X_t^L,Y_t^L,X_t^R,Y_t^R) \gets$ \Call{Partition}{$X_t,Y_t,j_t^\star,c_t^\star$}
      \State $\Delta_t \gets \Call{ImpurityDecrease}{Y_t,Y_t^L,Y_t^R,\theta}$
      \If{$\Delta_t$ is below the minimum required impurity decrease}
        \State \Return \Call{Leaf}{$Y_t$}
      \EndIf
      \State Add $\Delta_t$ to the feature-importance accumulator for $j_t^\star$
      \State $u \gets$ internal node with split $(j_t^\star,c_t^\star)$ and stored p-value fields $(q_t^{\mathrm{feat}},q_t^{\mathrm{split}})$
      \State $u.\mathrm{left} \gets \Call{GrowNode}{X_t^L,Y_t^L,d+1,F'_{\mathrm{avail}},\theta}$
      \State $u.\mathrm{right} \gets \Call{GrowNode}{X_t^R,Y_t^R,d+1,F'_{\mathrm{avail}},\theta}$
      \State \Return $u$
    \EndFunction
  \end{algorithmic}
\end{algorithm}

\subsection{Forests and runtime hyperparameters}
\label{sec:bootstrap}

CIF combines bootstrap trees and averages their predictions. For feature
selection, it sums each fitted tree's normalized split importances and
renormalizes the forest vector. Each CIF tree uses the same Stage~A and Stage~B
node rules as CIT
\citep{Breiman2001RandomForests,HothornBuehlmannDudoitMolinaroVanDerLaan2006SurvivalEnsembles}.
The fixed-node theorem applies only to the Stage~A reference test described
above.
Across the CIT/CIF benchmark configurations, we hold Bonferroni correction,
adaptive stopping, feature muting, feature scanning, threshold scanning,
histogram-256 thresholding, and CIF bootstrap sampling fixed. The varying
settings are selector choice and honesty.

Runtime depends on the realized work at each node: features tested in Stage~A,
thresholds generated and retained in Stage~B, permutations evaluated before
stopping, and features carried into the node. Adaptive stopping changes the
number of permutations evaluated for each test
\citep{BesagClifford1991SequentialMonteCarloPValues,Gandy2009SequentialImplementationMonteCarloTests}.
Histogram thresholding reduces the Stage~B threshold set by replacing exhaustive
cutpoints with binned thresholds. Feature and threshold scanning change the
order in which features or thresholds are tested under early stopping; feature
muting can reduce the feature pool inherited by descendants.
Appendix~\ref{app:methods} gives algorithm and
reproducibility details, Appendix~\ref{app:complexity-details} gives runtime and
memory accounting, and \cref{sec:results-practical} reports the measured timing
effects.

\subsection{Feature rankings}
\label{sec:rank-construction}

The experiments evaluate each method through the ranking it produces on the
training fold. CIT and CIF rankings are fitted split-importance rankings, not
rankings of Stage~A $p$-values. For a fitted CIT, feature $j$ receives the sum,
over internal nodes that split on $j$, of the impurity decrease at that node:
parent impurity minus the weighted child impurity. If the total importance is
positive, the vector is normalized to sum to one; if the tree contributes no
positive impurity decrease, the vector remains zero. CIF sums these tree
importance vectors over trees and then normalizes the forest vector in the same
way. Stage~A
determines which features enter threshold optimization, while impurity decrease
scores only the accepted splits by their local reduction in impurity. Larger
importances are ranked earlier, and features never used by a split receive zero
importance.

This split-importance ranking differs from standard random-forest impurity
importance in how a feature enters threshold search. In CART-style trees, the
feature and threshold are chosen together by an impurity search, so a feature
with more possible thresholds has more chances to produce a large apparent
impurity decrease \citep{Strobl2007BiasRFVariableImportance}. In CIT and CIF,
Stage~A first selects a feature by a permutation test; Stage~B then searches
thresholds only for that feature. This reduces the advantage that
high-cardinality features get from having more possible thresholds. We use these
fitted split-importance scores as rankings and evaluate them through downstream
top-$k$ prediction and synthetic feature recovery.

\section{Fixed-node Stage~A control}
\label{sec:theory}

This section analyzes the exhaustive fixed-$B$ Stage~A feature selection rule at a
fixed node. The theorem proves superuniformity for a single $+1$ Monte Carlo
permutation $p$-value, and the corollary applies Bonferroni correction across the
fixed-node Stage~A tests. In the notation of \cref{sec:stageA-stageB}, this is
the reference case where every feature in $F_t$ is tested with the same fixed
nodewise budget $B$. All probabilities below condition on $(X_t,U)$, where $U$
collects label-independent auxiliary randomness other than the Monte Carlo
permutation draws.

\begin{theorem}[$+1$ permutation $p$-values are superuniform]
\label{thm:plusone-superuniform}
Assume A0.1 through A0.5 from \cref{app:assumptions}. Under the nodewise
complete permutation null at the fixed node $t$, conditional on $(X_t,U)$, the
observed statistic together with the $B$ random-permutation statistics has the
rank distribution of an exchangeable tuple of length $B+1$. Define the
right-tail $+1$ Monte Carlo $p$-value
\begin{equation}
  p_{\mathrm{MC}}^{\ge}
  := \frac{1 + \sum_{b=1}^B \ind\{T_b \ge T_{\mathrm{obs}}\}}{B+1}.
\end{equation}
Then for every $\alpha \in [0,1]$,
\begin{equation}
  \PP(p_{\mathrm{MC}}^{\ge} \le \alpha \mid X_t,U) \le \alpha.
\end{equation}
\end{theorem}

This is the standard exchangeable-rank result for Monte Carlo permutation tests
\citep{NorthCurtisSham2002EmpiricalPValuesMonteCarlo,Ernst2004PermutationMethodsExactInference,HemerikGoeman2018ExactTestingRandomPermutations}.
\Cref{app:permutation-exchangeability,app:proof-plusone} state the
exchangeability condition and give the proof. The left-tail version follows by
applying the same result to $-T$.

\begin{corollary}[Stage~A complete null control at a fixed node]
\label{cor:stageA-global-null}
Assume A0.1 through A0.5 from \cref{app:assumptions}. Under the nodewise complete
null at $t$, let $F_t$ be nonempty and write $m_t = |F_t| \ge 1$. Let Stage~A
reject when
\begin{equation}
  \min_{j \in F_t} p_{t,j} \le \alpha_{\mathrm{sel}}/m_t.
\end{equation}
The tree-growing rule uses a strict comparison, so its rejection event is a
subset of the nonstrict event above.
Then
\begin{equation}
  \PP\!\left(\exists j \in F_t : p_{t,j} \le \alpha_{\mathrm{sel}}/m_t \,\middle|\, X_t,U\right)
  \le \alpha_{\mathrm{sel}}.
\end{equation}
\end{corollary}

Combining \cref{thm:plusone-superuniform} with the Bonferroni union bound over
the $m_t$ Stage~A tests gives fixed-node Stage~A control under the complete
null. The result is local: it does not cover partial-null familywise error
control, $p$-values for selected features, or error control for fitted trees and
forests.

\section{Experimental design}
\label{sec:experiments}

The experiments ask whether rankings produced by CIT and CIF place predictive
features early enough for fixed downstream learners. For each training fold, a
method produces a feature ranking, and fixed downstream learners evaluate the
selected features on the held-out fold. We then summarize overall ranks, how
often CIF performs well across datasets, high-dimensional behavior, recovery on
synthetic data with known informative features, CIF ranking ablations, and
feature sampling in forests.

\subsection{Benchmark setup}
\label{sec:experiments-protocol}

The benchmark uses 5 random seeds, with 5-fold cross-validation within each
seed. Classification fits logistic regression (LR), support vector machine
(SVM), and k-nearest neighbors (KNN) as downstream models and reports balanced
accuracy, defined as mean recall across classes. Regression fits ridge
regression, support vector regression (SVR), and KNN and reports $R^2$. The main
tables use the standard values $k \in \{5,10,25,50,100\}$.
Downstream learner settings are fixed. Logistic regression uses balanced class
weights and a 1,000-iteration limit; SVM uses an RBF kernel with unit
regularization, $\gamma=\texttt{scale}$, and balanced class weights; KNN uses
5 distance-weighted neighbors; ridge uses unit penalty strength; and SVR uses an
RBF kernel with unit regularization, $\epsilon=0.1$, and $\gamma=\texttt{scale}$.

For each seed and fold, feature-standardization parameters are estimated on the
training fold and then applied to the held-out fold. The ranking method uses only
the standardized training fold and produces the ordered feature list described in
\cref{sec:rank-construction}. For each $k$, we take the top-$k$ features from
that ranking, refit standardization on the selected training columns, apply the
same transform to the selected held-out columns, fit the downstream model on the
selected training data, and evaluate it on the held-out fold.

The downstream evaluation also records performance using the full feature set,
$k=p$, for every dataset. For datasets with more than 100 features, the separate
high-$p$ analysis adds fixed values $k \in \{150,200,300,500,750,1{,}000\}$
and fractional values
$\lceil 0.25p\rceil,\lceil 0.5p\rceil,\lceil 0.75p\rceil$, keeping only values
no larger than $p$. The main tables average only the standard $k$ values.

\subsection{Benchmark comparisons}
\label{sec:experiments-analysis-plan}

The benchmark averages scores in three steps. First, for each
$(\text{dataset}, \text{downstream model}, k)$ combination, scores are averaged
over folds and seeds. Second, each dataset score averages those fold-and-seed
means over downstream models and over the standard $k$ values available for the
methods in that comparison. Third, the tables average dataset scores over
datasets.

The real-data classification benchmark contains 23 datasets. Comparisons
involving partykit ctree and cforest use the 22-dataset subset that excludes
\nolinkurl{dexter}; comparisons among methods with \nolinkurl{dexter} results
use all 23 datasets. The complete-case classification analysis, used for
bootstrap intervals and omnibus rank testing, contains 14 datasets with all
17 methods, all three downstream learners, and all five standard values of $k$.
Regression contains 8 datasets; the analogous rank summary uses the
6 complete-case datasets with all methods, downstream learners, and standard
$k$ values available.

By-$k$ results are reported separately for $k=5,10,25,50,100$. The corresponding
classification dataset counts are $22, 21, 15, 15, 14$; the regression counts
are $8, 8, 7, 7, 6$. Pairwise CIF comparisons stratified by $k$ use the dataset
counts reported with the corresponding rows. The 14-dataset classification and
6-dataset regression summaries are complete-case analyses, retaining only
datasets with every method, downstream learner, and standard $k$ value.

\subsection{Datasets}
\label{sec:experiments-datasets}

The real-data benchmark includes 23 classification datasets and 8 regression
datasets. The separate high-$p$ analysis restricts attention to the
15 classification datasets and 6 regression datasets with more than 100
features. The full inventories of the real and synthetic datasets are listed in
\Cref{app:dataset-inventory}, which also summarizes the synthetic designs and
their parameter choices.

Feature recovery uses 8 synthetic classification datasets and 8 synthetic
regression datasets with known informative features and controlled noise
structure. A separate study of feature sampling uses $n=250$ and
varies $p \in \{100, 500, 1{,}000\}$ and
$n_{\mathrm{informative}} \in \{1,2,5,10\}$. That study compares CIT, a CART
decision tree (DT), a random threshold tree (RT), CIF, CIF-all, random forests
(RF), and ExtraTrees (ET), using 1,000 trees for forest methods. CIF-all denotes
the CIF variant that evaluates all nonconstant features at each split.

\subsection{Compared methods}
\label{sec:experiments-methods}

The benchmark includes three types of feature selectors: permutation-test
filters, embedded tree or ensemble methods, and wrapper baselines.
Permutation-test filters compute univariate dependence statistics and rank
features by their permutation $p$-values: multiple correlation (MC) and
randomized dependence coefficient (RDC) for classification, and Pearson
correlation (PC), distance correlation (DC), and RDC for regression. The RDC
filter keeps the empirical-CDF random feature map from
\citet{LopezPaz2013RDC} but replaces the canonical-correlation solve with the
maximum absolute pairwise correlation across projected columns. Distance
correlation follows \citet{Szekely2007DistanceCorrelation}.

Embedded methods rank features using tree or ensemble importance scores. These
methods are CART decision trees (DT), randomized trees (RT), CIT, CIF, random
forests (RF) \citep{Breiman2001RandomForests}, ExtraTrees
\citep{Geurts2006ExtraTrees}, XGBoost \citep{ChenGuestrin2016XGBoost},
LightGBM \citep{Ke2017LightGBM}, CatBoost
\citep{Prokhorenkova2018CatBoost}, and the partykit reference implementations
ctree and cforest
\citep{RPartykitPackage,Hothorn2006UnbiasedRecursivePartitioning,HothornZeileis2015Partykit}.
Wrapper baselines score features through auxiliary random-forest fits: Boruta
\citep{Kursa2010Boruta}, permutation importance (PI)
\citep{Breiman2001RandomForests,Strobl2007BiasRFVariableImportance},
conditional permutation importance (CPI)
\citep{DebeerStrobl2020ConditionalPermutationImportanceRevisited}, and
random-forest recursive feature elimination (RF-RFE)
\citep{Guyon2002RFE,Pedregosa2011ScikitLearn}.

In this benchmark, CPI uses stratified conditional permutation with
correlation-based strata. For each target feature, it trains a 100-tree random
forest on a training split and scores a 20\% validation split. It then
identifies validation-set features with absolute correlation at least 0.5 with
the target feature, forms five-bin percentile strata from those conditioning
features, and permutes the target feature within strata. If no conditioning
feature meets the correlation threshold, the target feature is permuted globally.

Ranking rules use each method's natural score direction. Embedded tree and
ensemble baselines rank features by decreasing fitted importance: DT, RT, RF,
ExtraTrees, XGBoost, LightGBM, and CatBoost. Permutation-test filters rank by
increasing $p$-value, with dependence score as the secondary key. The ctree
reference ranks features by split-use counts. For cforest, we use partykit's
unconditional variable-importance scores with one permutation; CIF uses
split-importance rankings from accepted conditional-inference splits. The
cforest results therefore refer to this benchmarked partykit importance setting.
Boruta and recursive feature elimination rank by their fitted rank variables,
with smaller ranks first. PI and CPI rank by decreasing importance; the CPI row
uses the stratified conditional-permutation procedure described above. Ties are
broken by deterministic, method-specific order for reproducibility.

\subsection{Configuration selection}
\label{sec:experiments-repro}

Several methods, including CIT, CIF, XGBoost, LightGBM, ctree, and cforest, were
evaluated under multiple configurations. Permutation-test filters used one configuration per
statistic. For each task, the paper reports one configuration per method,
selected by mean score over real datasets, downstream models, and the standard
values $k \in \{5,10,25,50,100\}$. Methods with only one configuration are
treated as fixed. The selected configuration is then held fixed across datasets for
that task. Because configuration selection and reporting use the same real-data
benchmark, the reported scores, intervals, and rank summaries are descriptive.
As a sensitivity check, we repeat configuration selection in leave-one-dataset-out
(LODO) form for both tasks. In classification, each of the 14 complete-case
datasets is evaluated using configurations chosen from the other 13 datasets. In
regression, each of the 6 complete-case datasets is evaluated using configurations
chosen from the other 5 datasets. The downstream models, $k$ values, and
averaging within each dataset are unchanged.

For the synthetic analysis, configurations are selected separately from the real-data
benchmark. Each method is assigned the configuration with the highest average recovery
of the true informative features over $k = 5, 10, 25, 50,$ and $100$. The
benchmarked CIT and CIF configurations use adaptive sequential stopping, Bonferroni
correction, feature muting, feature scanning, threshold scanning, and 256-bin
histogram thresholding. The benchmarked CIT and CIF variants differ in selector
and honesty; CIF is run as a bootstrap forest.
Each CIF runtime ablation changes one setting.

To isolate which parts of the CIF ranking matter, we run a ranking ablation on
the real-data benchmark datasets. The ablation holds the selected CIF
configuration fixed and changes one component at a time: replacing fitted split
importance with split counts, disabling feature muting, disabling bootstrap
sampling, or reducing the forest to one tree. Each variant is evaluated with the
same downstream learners, folds, seeds, and standard top-$k$ values as the main
benchmark.

\section{Results}
\label{sec:results}

\subsection{Real-data benchmark}
\label{sec:experiments-real}

\Cref{tab:method-main-comparison} reports the main real-data comparison.
Classification uses the 22 datasets with ctree and cforest results; the full
classification collection has 23 datasets, with \nolinkurl{dexter} reserved for
analyses that compare methods available on that dataset. Regression uses all 8
datasets. Within each dataset, scores are first averaged over downstream
learners and the standard $k$ values shared by the compared methods, then
converted to ranks. Lower mean rank is better.

\begin{table}[H]
  \centering
  \fontsize{8}{8.6}\selectfont
  \setlength{\tabcolsep}{5pt}
  \renewcommand{\arraystretch}{0.88}
  \begin{tabular}{@{}l S[table-format=2.2] @{\hspace{2.5em}} l S[table-format=2.2]@{}}
    \toprule
    \multicolumn{2}{c}{Classification} &
    \multicolumn{2}{c}{Regression} \\
    \cmidrule(lr){1-2}\cmidrule(l){3-4}
    \tablehead{Method} & {Mean rank} & \tablehead{Method} & {Mean rank} \\
    \midrule
    LightGBM & 5.17 & ExtraTrees & 6.56 \\
    XGBoost & 5.58 & CatBoost & 6.97 \\
    CatBoost & 5.85 & CIF & 7.07 \\
    CIF & 6.08 & RF & 8.09 \\
    RF & 6.27 & RT & 8.37 \\
    RF-RFE & 6.98 & DT & 8.68 \\
    ExtraTrees & 7.45 & LightGBM & 8.76 \\
    DT & 7.67 & RDC filter & 8.83 \\
    RT & 10.05 & CIT & 9.19 \\
    CIT & 10.34 & PC filter & 9.35 \\
    Boruta & 10.51 & DC filter & 9.77 \\
    MC filter & 10.65 & XGBoost & 9.96 \\
    PI & 10.76 & Boruta & 10.00 \\
    cforest & 11.36 & PI & 10.19 \\
    RDC filter & 12.29 & RF-RFE & 10.41 \\
    ctree & 12.86 & cforest & 12.20 \\
    CPI & 13.13 & ctree & 12.99 \\
    \multicolumn{2}{c}{} & CPI & 13.61 \\
    \bottomrule
  \end{tabular}
  \caption{Main real-data comparison. Ranks are computed within each dataset
  after scores are averaged over downstream learners and the standard $k$ values
  shared by the compared methods, then averaged across datasets. Classification
  uses the 22 datasets with ctree and cforest results; regression uses all 8
  datasets.}
  \label{tab:method-main-comparison}
\end{table}

With selected configurations fixed, CIF ranks 4th of 17 classification methods
and 3rd of 18 regression methods.
\Cref{tab:conditional-inference-comparisons} compares CIF directly with ctree,
cforest, and CIT.
\Cref{app:benchmark-stability} gives the learner-by-$k$ tables
(\cref{tab:classification-learner-k-sensitivity,tab:regression-learner-k-sensitivity}),
pairwise CIF-comparison heatmaps
(\cref{fig:benchmark-pairwise-sensitivity,fig:regression-benchmark-pairwise-sensitivity}),
and seed-sensitivity tables
(\cref{tab:classification-seed-sensitivity,tab:regression-seed-sensitivity});
each analysis reports its dataset count.

\begin{table}[H]
  \centering
  \scriptsize
  \setlength{\tabcolsep}{5pt}
  \renewcommand{\arraystretch}{0.92}
  \begin{tabular}{@{}llS[table-format=2.0]S[table-format=+1.3]c@{}}
    \toprule
    \tablehead{Task} & \tablehead{Compared with} & {Datasets} & {Mean diff.} & W--L \\
    \midrule
    Classification & ctree & 22 & +0.088 & \winloss{22}{0} \\
    Classification & cforest & 22 & +0.072 & \winloss{19}{3} \\
    Classification & CIT & 23 & +0.032 & \winloss{22}{1} \\
    Regression & ctree & 8 & +0.226 & \winloss{7}{1} \\
    Regression & cforest & 8 & +0.721 & \winloss{7}{1} \\
    Regression & CIT & 8 & +0.360 & \winloss{6}{2} \\
    \bottomrule
  \end{tabular}
  \caption{Direct CIF comparisons against ctree, cforest, and CIT. Mean
  differences are CIF minus the compared method after averaging within each
  dataset over supported downstream learners and standard values of $k$;
  classification uses balanced accuracy and regression uses $R^2$. W--L reports
  wins and losses over datasets where both methods were run.}
  \label{tab:conditional-inference-comparisons}
\end{table}

\FloatBarrier

\Cref{tab:conditional-inference-comparisons} shows CIF ahead on average in all
six direct pairings. On the complete-case datasets, CIF is also ahead of ctree,
cforest, and CIT in both tasks. The 95\% bootstrap CIs for mean
CIF-minus-comparator differences are as follows. Classification (14 datasets):
ctree [+0.068, +0.154], cforest [+0.067, +0.160], CIT [+0.023, +0.060]. Regression
(6 datasets): ctree [+0.113, +0.529], cforest [+0.040, +2.570], CIT
[-0.007, +1.224]. The regression interval for CIF versus CIT crosses zero; the
other complete-case intervals are positive. The regression intervals are wider
because fewer datasets are available and some comparator mean $R^2$ values are
negative.

\paragraph{CIF ablations.}
\Cref{tab:cif-ranking-ablation} compares the selected CIF configuration with
four targeted changes: split-count ranking, disabling bootstrap sampling,
disabling feature muting, and using one tree.

\begin{table}[H]
  \centering
  \scriptsize
  \setlength{\tabcolsep}{4pt}
  \renewcommand{\arraystretch}{1.08}
  \begin{tabular}{@{}llS[table-format=+1.3]cc@{}}
    \toprule
    \tablehead{Task} & \tablehead{CIF change} & {Mean diff.} & 95\% CI & W--L \\
    \midrule
    Classification & Split-count ranking & -0.003 & [-0.007, +0.001] & 10--13 \\
    Classification & Disable bootstrap & +0.001 & [-0.003, +0.004] & 10--12 \\
    Classification & Disable feature muting & +0.000 & [-0.001, +0.002] & 13--8 \\
    Classification & Use one tree & -0.061 & [-0.091, -0.035] & 1--22 \\
    \addlinespace
    Regression & Split-count ranking & -0.022 & [-0.060, +0.000] & 3--5 \\
    Regression & Disable bootstrap & -0.208 & [-0.604, -0.004] & 3--5 \\
    Regression & Disable feature muting & +0.000 & [-0.011, +0.009] & 5--3 \\
    Regression & Use one tree & -0.552 & [-1.328, -0.073] & 0--8 \\
    \bottomrule
  \end{tabular}
  \caption{CIF ablations on the real-data benchmark. Mean differences compare
  each changed variant with the selected CIF configuration using
  split-importance ranking. Scores are averaged within each dataset over
  downstream learners and supported standard values of $k$. Classification uses
  balanced accuracy; regression uses
  $R^2$. The confidence intervals are percentile bootstrap intervals over
  dataset-level differences from 20{,}000 replicates. Positive values favor the
  changed CIF variant. W--L counts datasets where the changed variant is higher
  or lower than the reference; exact ties are omitted. The classification rows
  use 23 datasets and the regression rows use 8 datasets.}
  \label{tab:cif-ranking-ablation}
\end{table}

\FloatBarrier

Reducing CIF to one tree has the largest effect: $-0.061$ balanced accuracy in
classification and $-0.552$ $R^2$ in regression. Replacing fitted split
importance with split counts changes the
mean score by $-0.003$ in classification and $-0.022$ in regression, with
intervals that include zero. Disabling feature muting changes both tasks by
$+0.000$, while disabling bootstrap sampling changes classification by $+0.001$
and regression by $-0.208$.

The complete-case Friedman tests use dataset-level method means.
Classification: $\chi^2(16)=141.32$, $p<0.001$, Kendall's $W=0.63$.
Regression: $\chi^2(17)=33.79$, $p=0.0089$, Kendall's $W=0.33$. The regression
panel has 6 datasets, so the $p$-value should be read with the effect sizes,
intervals, and LODO results. These intervals and rank tests summarize the
benchmark-selected configurations. The LODO analysis below checks whether those
configurations are stable under one-dataset holdout.

In the 14-dataset classification panel, LODO ranks CIF fifth of 17 on average
(mean rank 5.64), and CIF reselects the main-benchmark configuration in all 14
holdouts. In the 6-dataset regression panel, LODO ties CIF with CatBoost for
second on average (mean rank 5.50), behind ExtraTrees; CIF reselects the global
regression configuration in 5 of 6 holdouts.
\Cref{tab:cif-ci-comparisons,tab:lodo-config-sensitivity}
give the confidence intervals and LODO details.

\Cref{tab:cif-breadth-summary} reports how often CIF appears in the top half or
top 3 and how often it beats tree baselines.

\begin{table}[H]
  \centering
  \scriptsize
  \setlength{\tabcolsep}{6pt}
  \renewcommand{\arraystretch}{1.0}
  \begin{tabular}{@{}lcc@{}}
    \toprule
    \tablehead{Summary} & Classification & Regression \\
    \midrule
    Main rank & \countfrac{4}{17} & \countfrac{3}{18} \\
    Top half & \countfrac{21}{22} & \countfrac{6}{8} \\
    Top 3 & \countfrac{5}{22} & \countfrac{2}{8} \\
    Win vs ctree & \countfrac{22}{22} & \countfrac{7}{8} \\
    Win vs cforest & \countfrac{19}{22} & \countfrac{7}{8} \\
    Win vs CIT & \countfrac{22}{23} & \countfrac{6}{8} \\
    Win vs DT & \countfrac{14}{23} & \countfrac{6}{8} \\
    Win vs RT & \countfrac{21}{23} & \countfrac{5}{8} \\
    \bottomrule
  \end{tabular}
  \caption{Dataset breadth for the selected CIF configuration. Top half and
  Top 3 counts use the datasets from \cref{tab:method-main-comparison}. Direct
  win rows use datasets where both methods were run.}
  \label{tab:cif-breadth-summary}
\end{table}

In classification, CIF is in the top half on 21 of the 22 datasets with ctree
and cforest results, and in the top 3 on 5 of 22. In regression, CIF is in the
top half on 6 of 8 datasets and in the top 3 on 2 of 8. Against the greedy
single-tree baselines, CIF beats DT on 14/23 classification datasets and 6/8
regression datasets, and beats RT on 21/23 classification datasets. The RT
regression comparison is close: CIF is ahead on 5/8 datasets, with a mean
difference near zero.

\Cref{fig:k-trajectory,fig:regression-k-trajectory} report mean rank separately
by $k$. In classification, CIF's mean rank is 6.2 at $k=5$, 7.4 at $k=10$,
5.5 at $k=50$, and 5.0 at $k=100$.
In regression, CIF's mean rank is 8.9 at $k=5$, 5.6 at $k=50$, and 7.3 at
$k=100$. These by-$k$ results add context to
\Cref{tab:method-main-comparison}: CIF is strongest at larger displayed $k$
values in classification and at $k=50$ in regression. The aggregate ctree and
cforest comparisons use 22 classification datasets. For $k=5,10,25,50,100$,
the classification columns use 22, 21, 15, 15, and 14 datasets, respectively;
the regression columns use 8, 8, 7, 7, and 6 datasets, respectively.
Complete-case analyses use 14 classification datasets and 6 regression datasets.

\FloatBarrier

\begin{figure}[H]
  \centering
  \includegraphics[width=0.86\linewidth]{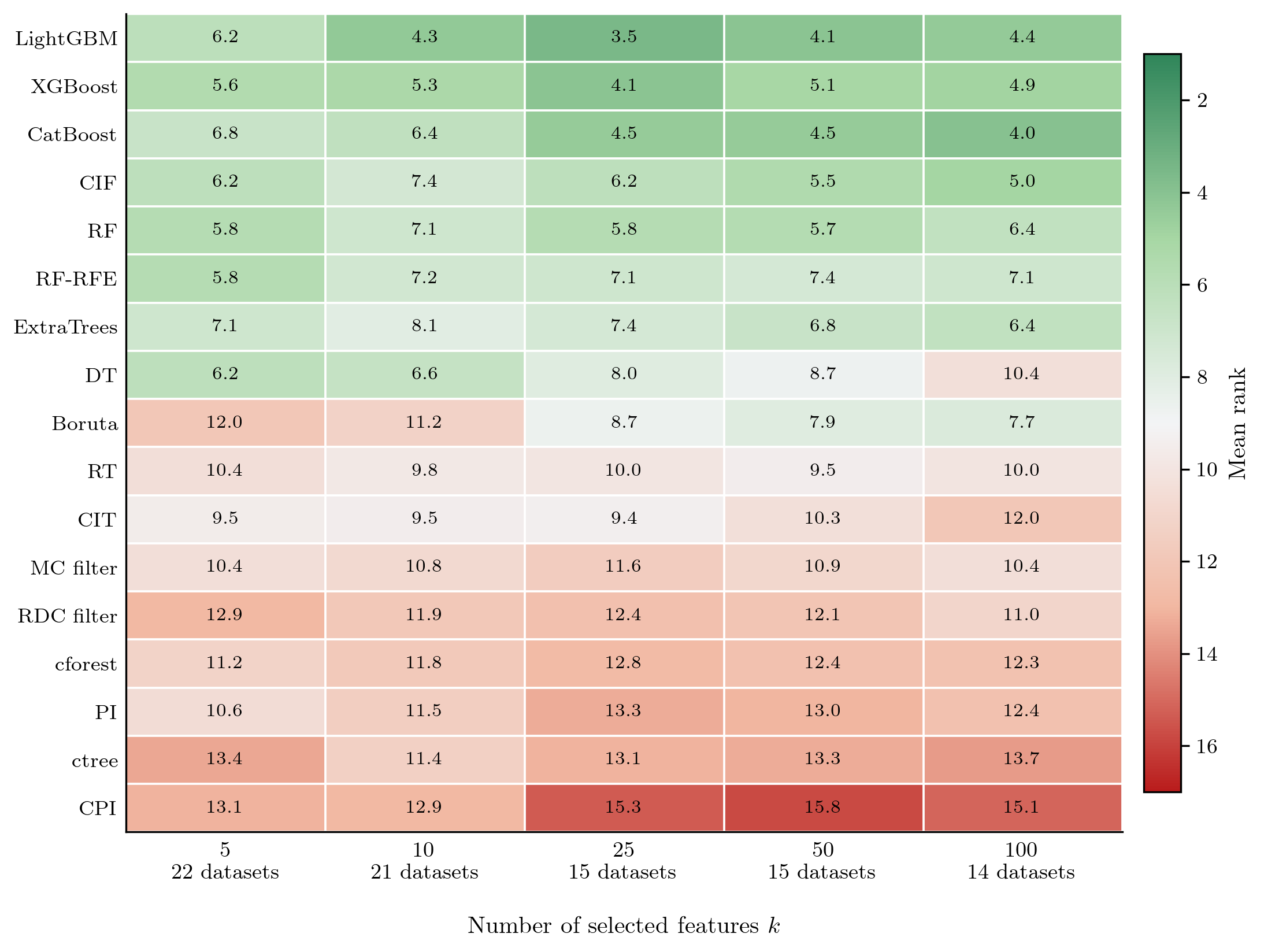}
  \caption{Classification mean rank by number of selected features. Rows show
  all 17 feature selection methods in the classification benchmark, averaged
  across downstream models at each value of $k$. Lower ranks are better.
  The columns for $k=5,10,25,50,100$ use 22, 21, 15, 15, and 14 datasets,
  respectively.}
  \label{fig:k-trajectory}
\end{figure}

\begin{figure}[H]
  \centering
  \includegraphics[width=0.86\linewidth]{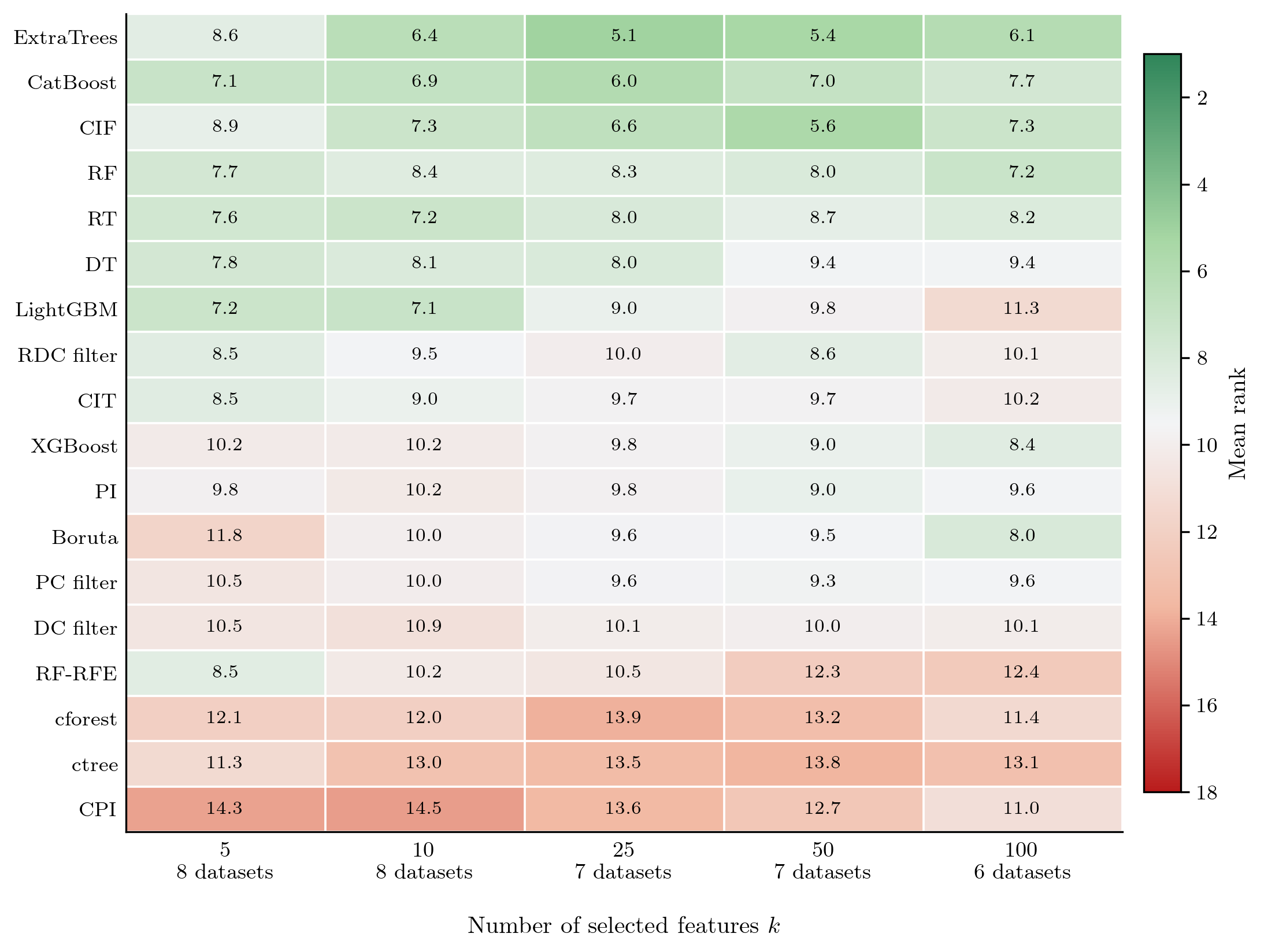}
  \caption{Regression mean rank by number of selected features. Rows show all
  18 feature selection methods in the regression benchmark, averaged across
  downstream models at each value of $k$. Lower ranks are better. The columns
  for $k=5,10,25,50,100$ use 8, 8, 7, 7, and 6 datasets, respectively.}
  \label{fig:regression-k-trajectory}
\end{figure}

\FloatBarrier

\subsection{Runtime ablations}
\label{sec:results-practical}

The runtime ablations identify which benchmark settings have the largest effect
on fitting time and report how much the ranking scores move when those settings
change. Each row changes one setting and reports elapsed wall-clock fit time
relative to the selected configuration for the same estimator. Ratios above 1 are
slower than that reference, and ratios below 1 are faster. CIT and CIF are shown
separately because the CIT
run records fit time and synthetic top-10 feature recovery, while the CIF runs
also record downstream score changes. The main benchmark remains the primary
performance comparison.
\Cref{app:complexity-details} links each ablation to the affected work terms.

\begin{table}[H]
  \centering
  \footnotesize
  \setlength{\tabcolsep}{6pt}
  \begin{tabular}{@{}lcc@{}}
    \toprule
    \tablehead{CIT change} & Runtime ratio & \shortstack{Largest top-10\\F1 change} \\
    \midrule
    Disable adaptive stopping & 0.18--1.14$\times$ & $\le 0.026$ \\
    Exact threshold search & 0.98--2.51$\times$ & $\le 0.008$ \\
    Disable feature scanning & 0.69--1.46$\times$ & $\le 0.017$ \\
    \bottomrule
  \end{tabular}
  \caption{CIT runtime ablations. Runtime ratios are paired by seed and dataset
  against the selected CIT configuration and summarized across the real and synthetic
  classification and regression runs. Top-10 F1 changes are absolute changes on
  synthetic datasets with known informative features, where top-10 F1 is the
  harmonic mean of precision and recall for recovering informative features among
  the first 10 ranked features. Numbers are rounded for
  display. The final column is the largest absolute top-10 F1 change across the
  synthetic task groups. The dataset-type counts are 7 real classification,
  8 synthetic classification, 2 real regression, and 6 synthetic regression.}
  \label{tab:cit-runtime-hyperparams}
\end{table}

\paragraph{CIT runtime ablations.}
\Cref{tab:cit-runtime-hyperparams} shows that CIT timing effects vary by data
regime. Disabling adaptive stopping reduces runtime in the synthetic runs and has
less consistent timing on real datasets. For a single fitted tree, changing this
option can also change the realized sequence of features, thresholds, and nodes
evaluated. These rows focus on timing and synthetic feature-recovery changes.

\begin{table}[H]
  \centering
  \footnotesize
  \setlength{\tabcolsep}{6pt}
  \begin{tabular}{@{}lccc@{}}
    \toprule
    \tablehead{CIF change} & Runtime ratio & \shortstack{Downstream\\score change} & \shortstack{Synthetic feature\\recovery change} \\
    \midrule
    Disable adaptive stopping & 4.0--8.4$\times$ & $\le 0.006$ & $\le 0.006$ \\
    Exact threshold search & 1.9--10.8$\times$ & $\le 0.011$ & $\le 0.017$ \\
    Disable feature scanning & 0.89--1.83$\times$ & $\le 0.006$ & $\le 0.010$ \\
    \bottomrule
  \end{tabular}
  \caption{CIF runtime ablations with score changes. Runtime ratios are relative to
  the selected CIF configuration. Adaptive stopping and feature scanning use median dataset-level
  ratios. The exact-threshold row compares group mean elapsed time with the
  histogram-256 reference. The score columns give the largest absolute downstream
  score change and synthetic feature-recovery change across the task groups
  used for each row. Inequality entries are rounded upward to the nearest 0.001.
  The task-group coverage matches \cref{tab:cit-runtime-hyperparams}.}
  \label{tab:cif-runtime-hyperparams}
\end{table}

\paragraph{CIF runtime ablations.}
\Cref{tab:cif-runtime-hyperparams} shows consistent slowdowns for two CIF
changes. Disabling adaptive stopping makes fits 4.0--8.4$\times$ slower by
median dataset-level ratios. In the separate threshold-search ablation, exact
threshold search is 1.9--10.8$\times$ slower than histogram-256 by ratios of
group mean elapsed times. Downstream score changes and synthetic
feature-recovery changes are at most 0.011 and 0.017. The feature-scanning effect
is smaller and less consistent once adaptive stopping is enabled. Removing
Bonferroni correction is faster for both CIT and CIF, but it changes the
statistical rule and yields trees with more selected features; the selected
setting keeps Bonferroni correction.
\FloatBarrier

\section{High-dimensional and synthetic analyses}
\label{sec:boundary}

\paragraph{High-dimensional real datasets.}
These analyses study where CIF's scores peak as $k$ grows, how feature
recovery on synthetic datasets changes across values of $k$, and how feature
sampling inside the forest affects how often splits use informative
features.
Across the 15 classification datasets with more than 100 features, CIF first
reaches its best score at an intermediate value of $k$ ($100 < k < p$) in 30 of 45
dataset-learner combinations. It first reaches its best score at the full
feature set ($k=p$) in only 5 of 45 combinations.
\Cref{fig:high-p-boundary} shows this pattern by downstream learner; regression
is summarized next. Classification scores most often peak at
intermediate $k$: 11 of 15 LR datasets, 9 of 15 SVM datasets, and 10 of 15 KNN
datasets. The full feature set is mainly an LR exception. Moving from $k{=}100$
to $k{=}p$ gives a mean $k{=}p$ minus $k{=}100$ score change of $+0.022$ for LR,
and $k=p$ matches the best observed LR score in 6 of 15 datasets. For SVM and
KNN, the corresponding mean changes are $-0.058$ and $-0.075$, and $k=p$ is
best in only 1 of 15 datasets for each learner.

In the 6 high-$p$ regression datasets, full-feature evaluation is rarely best.
Across the 18 dataset-learner combinations, 11 peak at or below $k{=}100$, only
2 peak at $k=p$, and the mean $k{=}p$ minus $k{=}100$ score change is $-0.308$.
The main full-feature exception remains the classification LR result in
\Cref{fig:high-p-boundary}.

\begin{figure}[H]
  \centering
  \includegraphics[width=0.84\linewidth]{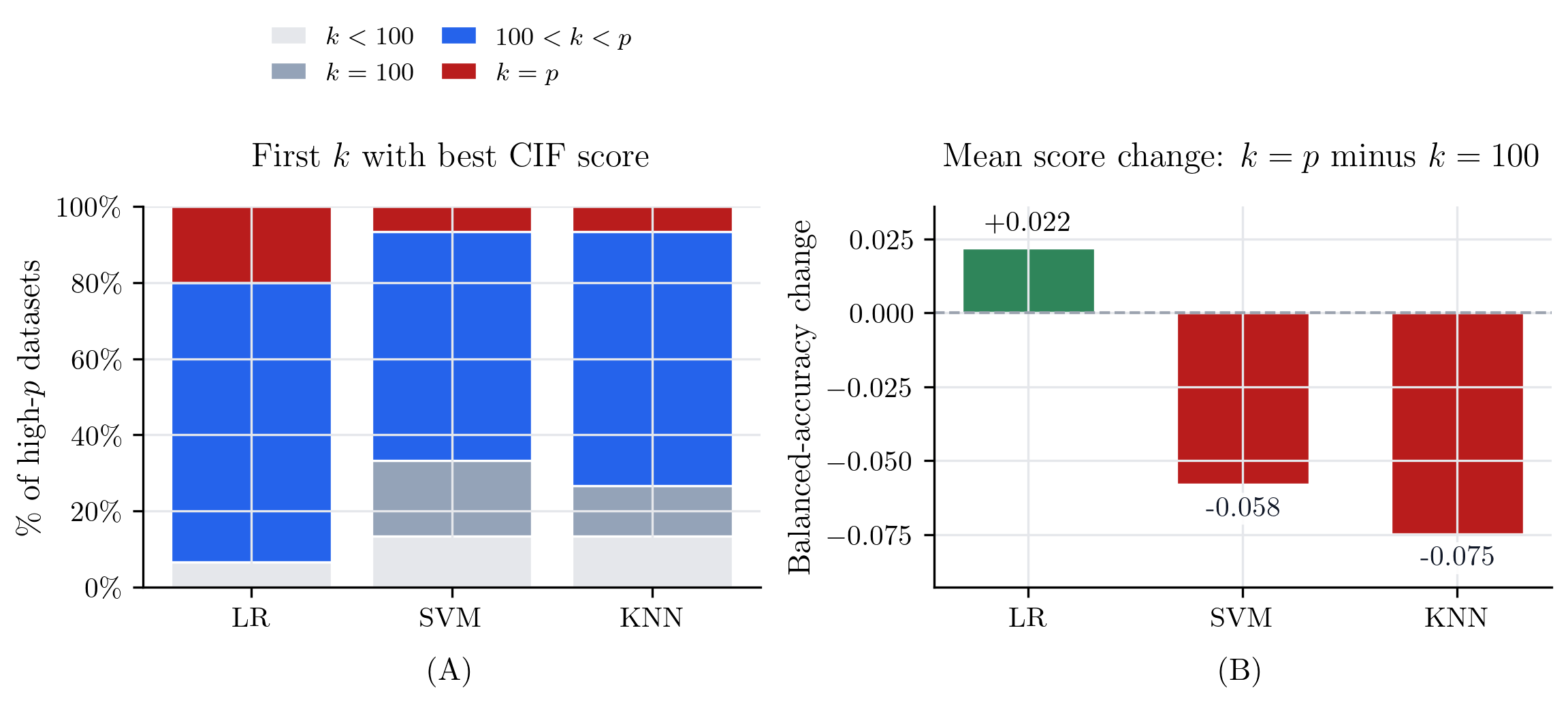}
  \caption{Where CIF's classification scores peak as $k$ increases, by downstream learner.
  (A) shows, for each learner, the percentage of 15 datasets where the first
  evaluated $k$ attaining the best observed CIF score falls below 100, at 100,
  between 100 and $p$, or at $k{=}p$. (B) shows the mean balanced accuracy
  change from $k{=}100$ to $k{=}p$.}
  \label{fig:high-p-boundary}
\end{figure}

\paragraph{Synthetic feature recovery.}
Across synthetic classification datasets, 34\% of the features in CIF's top-$k$
list are truly informative on average over the evaluated $k$ values, ranked
8th of 17, and its top-1 hit rate is 64\%, ranked 12th of 17. Across synthetic
regression datasets, the corresponding values are 39\%, ranked 10th of 18, and
81\%, ranked 15th of 18. In both tasks, CIF recovers true features within the
top-$k$ list more reliably than it places one first.
\Cref{tab:synthetic-head} summarizes this contrast, and
\Cref{fig:synthetic-topk-curves} shows recovery as $k$ grows. The percentage of
selected features that are informative falls as $k$ grows. On
redundant-feature synthetic designs, most selected features correspond to true
signals or redundant proxies: 83\% for classification and 84\% for regression.

\begin{table}[!htbp]
  \centering
  \footnotesize
  \setlength{\tabcolsep}{6pt}
  \begin{tabular}{@{}lcccc@{}}
    \toprule
    & \multicolumn{2}{c}{Classification} & \multicolumn{2}{c}{Regression} \\
    \cmidrule(lr){2-3}\cmidrule(l){4-5}
    \tablehead{Metric} & {Value} & Rank & {Value} & Rank \\
    \midrule
    Selected features that are informative & 34\% & \countfrac{8}{17} & 39\% & \countfrac{10}{18} \\
    Top-1 hit rate & 64\% & \countfrac{12}{17} & 81\% & \countfrac{15}{18} \\
    Selected features that are true or redundant & 83\% & \countfrac{5}{17} & 84\% & \countfrac{2.5}{18} \\
    \bottomrule
  \end{tabular}
  \caption{Synthetic feature recovery for CIF. The first row averages, over
  $k = 5, 10, 25, 50$, and $100$, the percentage of selected features that are
  truly informative. The top-1 hit rate is computed at $k=1$. The final row is
  computed and ranked only on the redundant-feature design for each task and
  counts true informative plus redundant features. Fractional ranks indicate
  averaged ties.}
  \label{tab:synthetic-head}
\end{table}

\begin{figure}[H]
  \centering
  \includegraphics[width=0.84\linewidth]{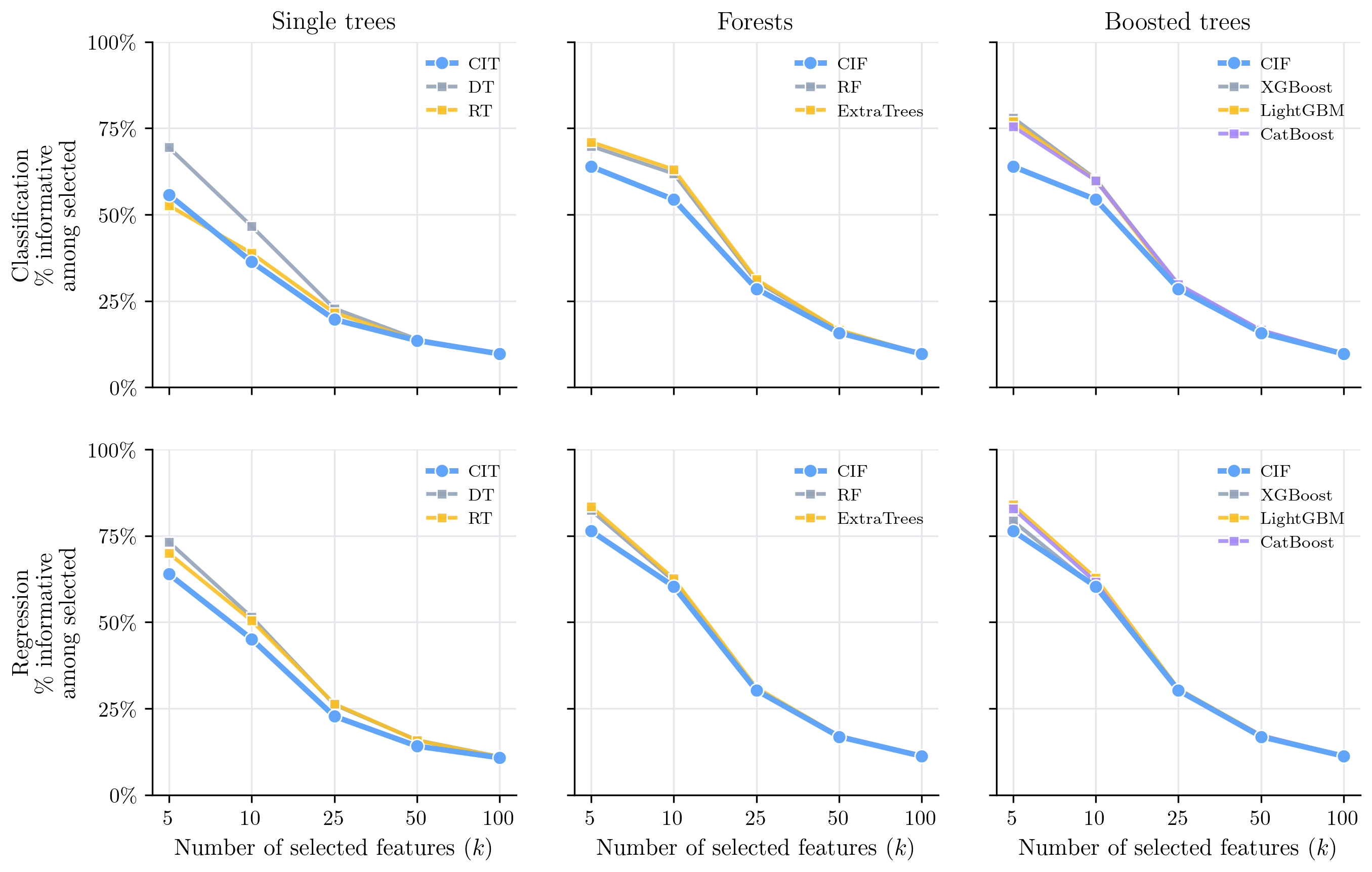}
  \caption{Synthetic top-$k$ recovery curves. Columns compare single trees,
  forests, and boosted trees; rows show classification and regression. Each
  value is the percentage of selected features that are truly informative.}
  \label{fig:synthetic-topk-curves}
\end{figure}

\paragraph{Feature sampling in forests.}
A sparse simulation asks whether forest feature sampling makes splits less
likely to use informative features in classification and regression settings.
\Cref{fig:candidate-forest-counts} shows one sparse classification
design comparing CIF with CIF-all. CIF samples a subset of features at each
split; CIF-all considers every nonconstant feature in Stage~A.
\Cref{app:candidate-availability} gives the classification and regression
results over $p \in \{100,500,1{,}000\}$.

In the single-tree simulations, CIT uses informative features in 100.0\% of
classification splits and 93.5\% of regression splits. In this fixed design,
informative features are usually selected when every nonconstant feature is
considered. In forests, sampling a subset of features reduces how often informative
features are used for splits. Across the classification forest
simulations with 1,000 trees, CIF uses informative features in 33.2\% of
splits, compared with 91.7\% for CIF-all, 8.4\% for RF, and 6.5\% for ET.
CIF also uses more distinct noise features than CIF-all, 265 versus 93
on average. For the $p{=}1{,}000$ classification setting with two informative
features, CIF uses informative features in 9.0\% of splits, while CIF-all uses
them in 100.0\%. In regression, CIF-all uses informative features more often
than CIF, 42.3\% versus 27.8\%, but uses more distinct noise features, 528.8
versus 367.7.
CIF-all isolates feature subsampling because feature muting and
the other tree settings are held fixed. In these sparse high-$p$ simulations,
splits use informative features most often when every nonconstant feature can
enter Stage~A. When the sampled feature subset rarely includes an informative
feature at a node, Stage~A selects from that subset.

\begin{figure}[!t]
  \centering
  \includegraphics[width=0.84\linewidth]{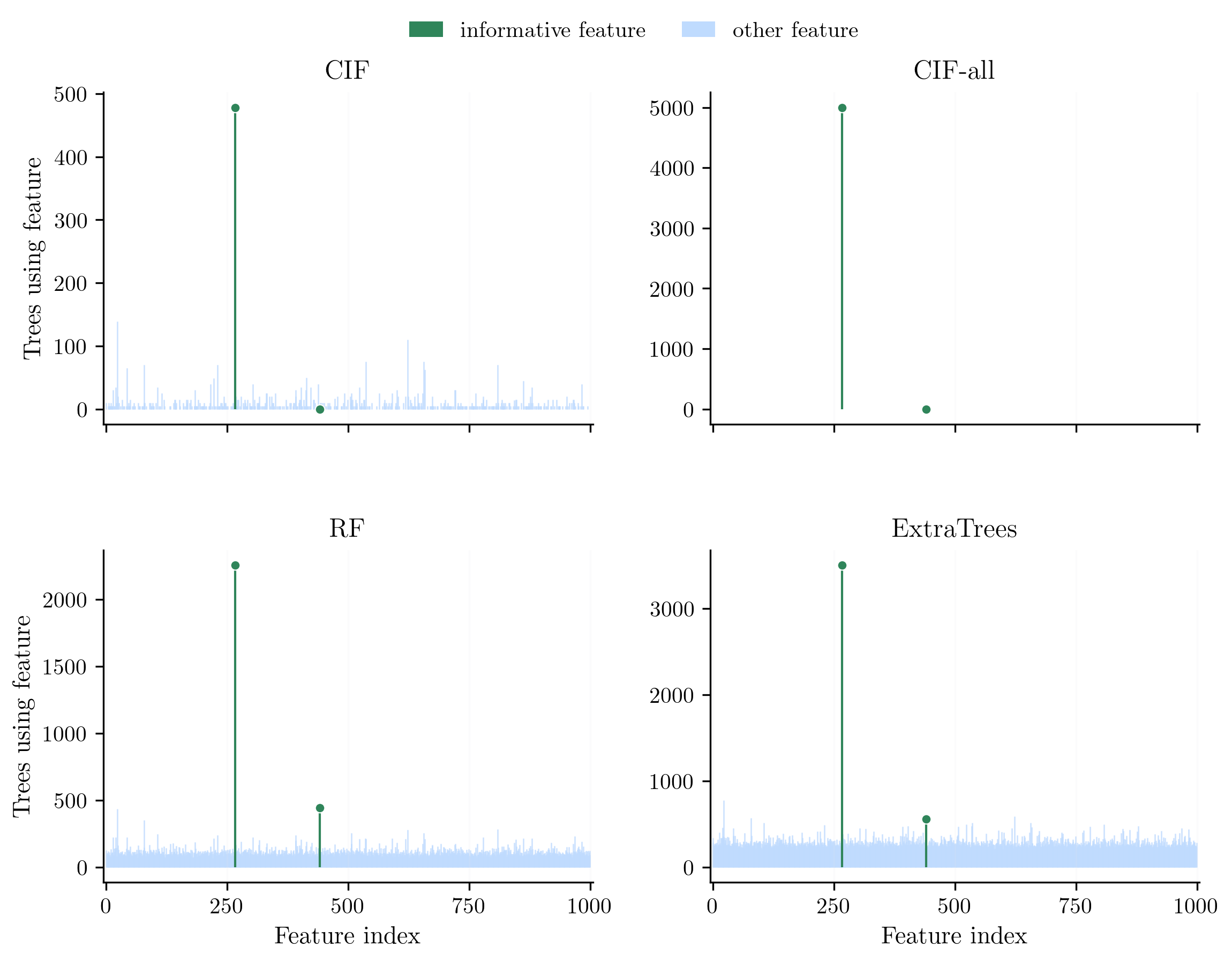}
  \caption{Tree-level feature use in one sparse classification forest design
  ($n{=}250$, $p{=}1{,}000$, two informative features, 1,000 trees). Counts are
  the number of trees using each feature in at least one split. CIF samples a
  subset of features; CIF-all evaluates all nonconstant
  features. Each subplot has its own
  y-axis range because counts differ across methods.}
  \label{fig:candidate-forest-counts}
\end{figure}

\FloatBarrier

\section{Discussion}
\label{sec:discussion}

In these experiments, CIF is best supported as a top-$k$ feature-selection
method for a downstream model. The benchmark asks whether a ranking places
predictive features early enough for fixed downstream learners. Classification
gives the clearest support; the smaller regression panel is positive but less
precise. The fixed-node Stage~A theorem provides the local guarantee; fitted
tree and forest rankings are assessed empirically.

\paragraph{Guidance.}
A CIF ranking should be judged through the downstream model that will use the
selected features. In practice, choose $k$ on held-out data over a small set of
values that includes the intended feature budget and, when feasible, the full
feature set as a sanity check. When prediction accuracy is the main target, tree-ensemble
baselines such as random forests, ExtraTrees, and boosted trees should remain in
the comparison set. When the target is variable screening or a ranked feature
list for later modeling, CIF should also be compared with ctree, cforest, and
forest-based screening methods such as Boruta, permutation or conditional
permutation importance, and RF-RFE
\citep{GenuerPoggiTuleauMalot2010VariableSelectionUsingRandomForests}.

The ablations show that adaptive stopping and histogram-256 thresholding account
for the largest measured CIF runtime changes, with small downstream score changes
in these experiments. Bonferroni correction is part of the Stage~A testing rule
and is not treated as a runtime shortcut. For ranking quality, forest size
matters more than the split-count and feature-muting variants tested here, so the
selected CIF configuration is the relevant ranking reference when ranking
quality is the priority.

In sparse high-$p$ settings, feature use should be monitored directly.
The simulations show that CIF feature sampling can leave informative features
out of many forest split decisions, especially in sparse classification designs.
Useful diagnostics include top-$k$ stability across seeds, the share of splits
using the final top-$k$ features or any prespecified positive control features,
the number of distinct low-ranked or known null features used in splits, and
sensitivity to larger sampled feature sets or to a small CIF-all run with fewer
seeds or trees.

\paragraph{Limitations.}
The fixed-node result is a reference guarantee for exhaustive fixed-$B$ Stage~A
testing at one node, with the tested features and permutation budget fixed
independently of the node responses. The benchmarked trees and forests add
adaptive stopping, feature and threshold ordering, feature muting, bounded threshold sets,
bootstrap samples, and forest aggregation, so their behavior is assessed through
the ranking benchmark.

The benchmark coverage is uneven across analyses. Some comparisons use broader
panels where both methods are available, while complete-case rank and interval
summaries use smaller panels that include every method. These summaries are
therefore descriptive rather than confirmatory.

The sampled-forest analysis isolates a sparse high-$p$ setting in which
forest feature sampling can limit how often splits use informative features:
larger sampled feature sets increase informative-feature use, but they also
change runtime and the mix of noise features used by splits. Track how often
splits use informative features alongside downstream performance and runtime
before changing the number of features sampled at each split.

\section{Conclusion}
\label{sec:conclusion}

This paper evaluates conditional-inference tree and forest rankings for top-$k$
downstream prediction and proves a fixed-node guarantee for Stage~A permutation
testing. The experiments support CIF as a top-$k$ feature-ranking method,
especially in classification: its rankings compare favorably with ctree,
cforest, and CIT, and they remain competitive among stronger tree-ensemble
baselines.

The ablations separate runtime and ranking effects. Adaptive stopping and
bounded threshold search avoid large runtime increases, while the one-tree
ablation has the largest ranking-quality loss. The high-$p$ and synthetic
analyses add two diagnostic cautions: sparse forests should be checked for
whether informative features are used in enough splits to enter the top-$k$
ranking, and CIF should not be treated as a first-feature discovery method when
the goal is to identify a single driver. Overall, these
results support CIF as a feature-ranking method for downstream prediction,
backed by a local Stage~A result and benchmark evidence.

\bibliographystyle{plainnat}
\bibliography{references}

\clearpage
\begin{appendices}
\crefalias{section}{appendix}
\crefalias{subsection}{appendix}
\section{Setup and assumptions}

\subsection{Fixed-node setup}

Let $(X_i, Y_i)_{i=1}^n$ be the training data. Fix a node $t$ with sample index
set $I_t$ and size $n_t:=|I_t|$. Write $X_t := (X_i)_{i\in I_t}$ and
$Y_t := (Y_i)_{i\in I_t}$ for the data restricted to that node. Stage~A tests
features $F_t$ with $m_t:=|F_t|$.

The fixed-node theorem analyzes the exhaustive fixed-$B$ Stage~A reference test.
Conditional on the node covariates and label-independent auxiliary randomness
$U$, $F_t$ and $B$ are fixed, and every feature in $F_t$ is tested with exactly
$B$ permutations.

\subsubsection*{Notation}
\cref{tab:fixed-node-notation} lists the notation used in the proof.
\begin{table}[H]
  \centering
  \small
  \setlength{\tabcolsep}{5pt}
  \begin{tabular}{@{}L{0.18\linewidth}L{0.76\linewidth}@{}}
    \toprule
    \tablehead{Symbol} & \tablehead{Meaning} \\
    \midrule
    $t$ & Node index, held fixed in the theorem. \\
    $I_t,\ n_t$ & Node sample index set and size $n_t:=|I_t|$. \\
    $X_t,\ Y_t$ & Covariates and labels restricted to node $t$. \\
    $U$ & Label-independent auxiliary randomness conditioned on in the proof; it excludes node responses, Monte Carlo permutation draws, permutation statistics, and $p$-values. \\
    $F_t,\ m_t$ & Stage~A features and size $m_t:=|F_t|$. \\
    $X_{t,j}$ & Values of feature $j$ on the node samples. \\
    $B$ & Number of random label permutations per tested feature in the fixed-$B$ reference test. \\
    $T^{\mathrm{sel}}_j$ & Stage~A selector statistic for feature $j$: $T^{\mathrm{sel}}_j(X_{t,j},Y_t;U)$. \\
    $T(\cdot,\cdot;U)$ & Generic scalar test statistic used in the exchangeability proof. \\
    $\Pi_b$ & Random label permutation used for Monte Carlo replicate $b$. \\
    $T_{\mathrm{obs}},T_1,\dots,T_B$ & Observed and permuted statistics for one fixed feature test. \\
    $R$ & Right-tail rank numerator $1+\sum_{b=1}^B \ind\{T_b \ge T_{\mathrm{obs}}\}$. \\
    $p_{\mathrm{MC}}^{\ge}$ & Right-tail $+1$ Monte Carlo permutation $p$-value for one fixed feature test. \\
    $p_{t,j}$ & Fixed-$B$ $+1$ Monte Carlo permutation $p$-value for feature $j$. \\
    $\alpha_{\mathrm{sel}}$ & Stage~A nominal level; the corollary tests at $\alpha_{\mathrm{sel}}/m_t$. \\
    \bottomrule
  \end{tabular}
  \caption{Notation used in the fixed-node Stage~A theorem and proof.}
  \label{tab:fixed-node-notation}
\end{table}

The guarantee applies only to fixed-node Stage~A. It does not cover Stage~B
threshold search, adaptive stopping, response-dependent feature muting,
response-selected internal nodes, bootstrap forests, honesty, or
$p$-values after feature selection or other post-selection steps; see
\cref{sec:method,app:methods,app:complexity-details}.

\subsection{Assumptions and theorem scope}
\label{app:assumptions}

\paragraph{A0.1 (Fixed node).}\label{assump:fixed-node}
Assumptions A0.1--A0.5 define the fixed-node reference setting used by
Theorem~\ref{thm:plusone-superuniform} and
Corollary~\ref{cor:stageA-global-null}.
The node $t$ and its sample index set $I_t$ are held fixed for the result. They
are not chosen using node responses or earlier response-dependent splits.

\paragraph{A0.2 (Complete nodewise permutation null).}\label{assump:exchangeability}
Under the complete null at node $t$, $Y_t$ is exchangeable conditional on
$(X_t,U)$: its conditional law is unchanged by permuting the node responses.
This complete nodewise permutation null does not cover a null feature when any
tested or untested node covariate remains associated with $Y_t$.

\paragraph{A0.3 (Features fixed before seeing node responses).}\label{assump:label-independence}
The features in $F_t$ and the effective nodewise budget $B$ are measurable
functions of $(X_t,U)$ only. They do not depend on $Y_t$, other training
responses used to reach the node, Monte Carlo permutation draws or statistics,
$p$-values, or sequential stopping decisions.

\paragraph{A0.4 (Fixed-$B$ computation).}\label{assump:fixed-b}
Each $p_{t,j}$ is computed with exactly $B$ random permutations, with no optional
stopping. The Monte Carlo permutation draws used to compute the $p$-values are
additional randomness, independent of $(X_t,Y_t,U)$, and are not included in
$U$.

\paragraph{A0.5 (Conservative ties).}\label{assump:ties}
For right-tail tests, the Stage~A $+1$ Monte Carlo $p$-value counts ties
conservatively using $\ind\{T_b \ge T_{\mathrm{obs}}\}$. If several features tie
after testing, any random tie break is label-independent and does not affect the
Stage~A reject/no-reject event.

\subsection{Results using these assumptions}

Both fixed-node results in \cref{sec:theory} use Assumptions A0.1--A0.5.
Theorem~\ref{thm:plusone-superuniform} proves $+1$ Monte Carlo permutation
$p$-value superuniformity. Corollary~\ref{cor:stageA-global-null} gives Stage~A
complete-null control.

\section{Monte Carlo permutation exchangeability}
\label{app:permutation-exchangeability}

The proof separates two facts. First, statistics computed from i.i.d. random
permutations are exchangeable. Second, under A0.2--A0.4, the usual ``observed
statistic + $B$ random permutations'' computation has the same null rank
distribution as a construction with a randomized baseline. The superuniformity
proof then uses the conservative tie convention in A0.5; see also
\citet{HemerikGoeman2018ExactTestingRandomPermutations} for a broader treatment
of exact random permutation testing.

First consider a fully randomized tuple. Fix a node $t$, feature $j$, and a
real-valued statistic $T(\cdot, \cdot; U)$ with a right tail convention (larger
is more extreme), where $U$ denotes label-independent auxiliary randomness. Let
$\Pi_0, \Pi_1, \dots, \Pi_B$ be i.i.d.\ uniform random permutations of
$\{1,\dots,n_t\}$, independent of the data (and of $U$), and define
\begin{equation}
  T_b := T(X_{t,j}, \Pi_b(Y_t); U) \qquad (b=0,1,\dots,B).
\end{equation}

\begin{lemma}[Exchangeability of Monte Carlo permutation statistics]
\label{lem:mc-exchangeability}
Conditional on $(X_t, Y_t, U)$, the vector $(T_0, T_1, \dots, T_B)$ is
exchangeable. Averaging over $Y_t$ preserves exchangeability, so the vector is
also exchangeable conditional on $(X_t,U)$ and unconditionally.
\end{lemma}

\begin{proof}
Conditional on $(X_t, Y_t, U)$, each $T_b$ is a measurable function of $\Pi_b$, and
$(\Pi_0,\dots,\Pi_B)$ are i.i.d. Thus any permutation of the indices
$b \in \{0,\dots,B\}$ leaves the joint law unchanged.
\end{proof}

\begin{lemma}[The observed statistic has the same null rank distribution]
\label{lem:identity-vs-random}
Assume A0.2--A0.4 (\cref{app:assumptions}). Let $\Pi_1,\dots,\Pi_B$ be i.i.d.\
uniform random permutations of $\{1,\dots,n_t\}$, independent of $(X_t,Y_t,U)$.
Define the observed statistic and permuted statistics by
\begin{equation}
  T_{\mathrm{obs}} := T(X_{t,j}, Y_t; U),
  \qquad
  T_b := T(X_{t,j}, \Pi_b(Y_t); U) \quad (b=1,\dots,B),
\end{equation}
and the standard +1 Monte Carlo $p$-value (right tail) by
\begin{equation}
  p_{\mathrm{id}}
  := \frac{1 + \sum_{b=1}^B \ind\{T_b \ge T_{\mathrm{obs}}\}}{B+1}.
\end{equation}
Let $\Pi_0$ be an additional independent uniform random permutation and set
$T_0 := T(X_{t,j}, \Pi_0(Y_t); U)$ (keeping $T_1,\dots,T_B$ as above), and define
\begin{equation}
  p_{\mathrm{rnd}}
  := \frac{1 + \sum_{b=1}^B \ind\{T_b \ge T_0\}}{B+1}.
\end{equation}
Both quantities use the same conservative right-tail tie convention: ties with
the baseline statistic are counted as exceedances. Under the null, conditional
on $(X_t,U)$, $p_{\mathrm{id}}$ and $p_{\mathrm{rnd}}$ have the same
distribution with this convention. It is enough to prove the rank bound for
the exchangeable construction $p_{\mathrm{rnd}}$ and apply the same conclusion
to the standard ``observed statistic + $B$ random permutations'' computation.
\end{lemma}

\begin{proof}
Fix $(X_t,U)$ and work under the null. By A0.2, permuting the node labels does
not change the conditional distribution of $Y_t$. Use the permutation-action
convention under which
$(\Pi_b \circ \Pi_0^{-1})(\Pi_0(Y_t)) = \Pi_b(Y_t)$. Let $\Pi_0$ be an
independent uniform permutation and set $Y_t' := \Pi_0(Y_t)$. Then
$Y_t' \stackrel{d}{=} Y_t$ conditional on $(X_t,U)$.

For $b=1,\dots,B$, define $\Pi_b' := \Pi_b \circ \Pi_0^{-1}$, so
$\Pi_b'(Y_t') = \Pi_b(Y_t)$. Conditional on $(Y_t,\Pi_0)$, right composition by
$\Pi_0^{-1}$ maps the independent uniform permutations $\Pi_b$ to independent
uniform permutations $\Pi_b'$. This conditional law does not depend on the
conditioning values. Because $Y_t'$ is a function of $(Y_t,\Pi_0)$, the
transformed permutations are independent of $(Y_t',\Pi_0)$.
Therefore,
\begin{equation}
  \begin{aligned}
    \big(T_{\mathrm{obs}}, T_1,\dots,T_B\big)
    &= \big(T(X_{t,j},Y_t;U), T(X_{t,j},\Pi_1(Y_t);U),\dots,T(X_{t,j},\Pi_B(Y_t);U)\big) \\
    &\stackrel{d}{=} \big(T(X_{t,j},Y_t';U), T(X_{t,j},\Pi_1'(Y_t');U),\dots,T(X_{t,j},\Pi_B'(Y_t');U)\big) \\
    &= \big(T_0, T_1,\dots,T_B\big),
  \end{aligned}
\end{equation}
where $\stackrel{d}{=}$ is conditional on $(X_t,U)$. Applying the same
$\ge$ comparison map to these equal-in-distribution vectors gives the same
conditional distribution for $p_{\mathrm{id}}$ and $p_{\mathrm{rnd}}$.
\end{proof}

\section{Proof of Theorem~\ref{thm:plusone-superuniform}}
\label{app:proof-plusone}

\Cref{app:permutation-exchangeability} reduces the usual computation with one
observed statistic and $B$ random permutations to the exchangeable rank
construction. We prove Theorem~\ref{thm:plusone-superuniform} under A0.1--A0.5
(\cref{app:assumptions}).

Fix $(X_t,U)$, fix one feature $j \in F_t$, and work under the
nodewise complete permutation null at the fixed node $t$. Under A0.2--A0.4, the
standard fixed-$B$ computation has the same conditional $p$-value distribution
as the randomized exchangeable construction; see
\cref{lem:identity-vs-random} in \cref{app:permutation-exchangeability} and
\citet{HemerikGoeman2018ExactTestingRandomPermutations}. It is enough
to prove the rank bound for an exchangeable tuple, which we denote by
$(T_{\mathrm{obs}},T_1,\dots,T_B)$.

Define the conservative right-tail rank numerator
\begin{equation}
  R := 1 + \sum_{b=1}^B \ind\{T_b \ge T_{\mathrm{obs}}\},
  \qquad
  p_{\mathrm{MC}}^{\ge} = \frac{R}{B+1}.
\end{equation}
The left-tail version replaces the exceedance check by
$\ind\{T_b \le T_{\mathrm{obs}}\}$; equivalently, apply the right-tail argument
below to $-T$. Stage~B uses this left-tail convention for split scores after
Stage~A selects a feature, but the fixed-node Type I error guarantee here
concerns Stage~A.

\paragraph{No ties.}
If the statistics almost surely have no ties, then the $\ge$ indicator is the
same as $>$, and $R$ is the descending rank of $T_{\mathrm{obs}}$ among
$(T_{\mathrm{obs}},T_1,\dots,T_B)$, with rank 1 assigned to the largest
statistic. Exchangeability implies
\begin{equation}
  \PP(R = r \mid X_t,U) = \frac{1}{B+1}, \qquad r=1,\dots,B+1,
\end{equation}
so for any $\alpha \in [0,1]$,
\begin{equation}
  \PP(p_{\mathrm{MC}}^{\ge} \le \alpha \mid X_t,U)
  = \PP\!\big(R \le (B+1)\alpha \mid X_t,U\big)
  = \frac{\lfloor (B+1)\alpha \rfloor}{B+1}
  \le \alpha.
\end{equation}

\paragraph{Ties.}
If ties can occur, A0.5 uses conservative tie handling: the right-tail $+1$
$p$-value uses $\ind\{T_b \ge T_{\mathrm{obs}}\}$, so ties with the observed
statistic are counted in the numerator. To compare this rule with a no-ties
rank, let
$V_{\mathrm{obs}},V_1,\dots,V_B \stackrel{\mathrm{i.i.d.}}{\sim}
\mathrm{Unif}(0,1)$ be independent of all other randomness and define augmented
statistics $\tilde{T}_{\mathrm{obs}} := (T_{\mathrm{obs}}, V_{\mathrm{obs}})$
and $\tilde{T}_b := (T_b, V_b)$ for $b=1,\dots,B$. Order the augmented pairs
lexicographically for the right tail: $\tilde{T}_b$ exceeds
$\tilde{T}_{\mathrm{obs}}$ exactly when $T_b > T_{\mathrm{obs}}$, or when
$T_b = T_{\mathrm{obs}}$ and $V_b > V_{\mathrm{obs}}$. Then
$(\tilde{T}_{\mathrm{obs}},\tilde{T}_1,\dots,\tilde{T}_B)$ is exchangeable and
almost surely has no ties, so the no-ties argument gives a superuniform
$p$-value
\begin{equation}
  \tilde{p}_{\mathrm{MC}}^{\ge}
  :=
  \frac{
    1 + \sum_{b=1}^B
    \ind\{
      T_b > T_{\mathrm{obs}}
      \text{ or } (T_b = T_{\mathrm{obs}} \text{ and } V_b > V_{\mathrm{obs}})
    \}
  }{B+1}.
\end{equation}
Moreover, each tie-broken indicator in $\tilde{p}_{\mathrm{MC}}^{\ge}$ is
bounded by the corresponding conservative indicator
$\ind\{T_b \ge T_{\mathrm{obs}}\}$. Thus
$p_{\mathrm{MC}}^{\ge} \ge \tilde{p}_{\mathrm{MC}}^{\ge}$, and
\begin{equation}
  \PP(p_{\mathrm{MC}}^{\ge} \le \alpha \mid X_t,U)
  \le \PP(\tilde{p}_{\mathrm{MC}}^{\ge} \le \alpha \mid X_t,U) \le \alpha.
\end{equation}
Taking expectations over $(X_t,U)$ gives the unconditional bound,
$\PP(p_{\mathrm{MC}}^{\ge} \le \alpha) \le \alpha$.
The fixed-node Stage~A corollary applies this single-test bound to the $m_t$
feature tests with threshold $\alpha_{\mathrm{sel}}/m_t$ and uses the Bonferroni
union bound; no independence among feature tests is required.

\section{Algorithm and reproducibility notes}
\label{app:methods}

Algorithm settings, benchmark configuration, and coverage rules determine how
the main results should be read. Section~\ref{sec:method}
describes the tree-growing loop, the setup and proof appendices state the
fixed-node results, and Appendix~\ref{app:complexity-details} gives runtime and
memory details for the benchmarked tree-growing algorithm.

\subsection{Scope relative to the theory}

The fixed-node theorem covers exhaustive fixed-$B$ Stage~A feature selection at
a fixed node under the complete permutation null. In that result, the tested
features and resample budget are fixed independently of the node responses. The
benchmarked tree-growing algorithm also uses settings outside that theorem:
\begin{itemize}
  \item early stopping,
  \item feature and threshold ordering inside early stopping,
  \item response-dependent feature muting,
  \item bounded threshold sets,
  \item bootstrap resampling and forest aggregation,
  \item optional honesty.
\end{itemize}
Under these conditions, the theorem covers exhaustive fixed-$B$ Stage~A values.
Stage~A and split scores from adaptive trees are tree-construction scores.

\subsection{Node-level choices}

\paragraph{Stage~A.}
At each node, the tree-growing rule removes features that are constant on
the node before any feature subsampling. Stage~A computes $+1$ Monte
Carlo $p$-values. It may scan features in preliminary-score order and, under
early stopping, may stop once a feature meets the nodewise threshold. With
feature muting, features with non-rejecting Stage~A values can be removed from
descendant feature pools in the same subtree. These shortcuts reduce
computation. When they are active, the returned $q_t^{\mathrm{feat}}$ is the
feature-selection score used to grow the tree; the exhaustive fixed-$B$
reference value is defined over the full feature pool.

\paragraph{Stage~B.}
After Stage~A selects a feature, Stage~B builds thresholds from
midpoints between consecutive unique feature values. Exact search tests all
midpoints. Random sampling tests a bounded subset; percentile and histogram
rules test bounded representative thresholds from the midpoint distribution.
With sequential stopping, Stage~B may scan thresholds in a reordered sequence.
Because Stage~B tests thresholds only after a response-selected feature is
chosen, $q_t^{\mathrm{split}}$ is a split-selection score.

\paragraph{Split statistics.}
For permutation split testing, Stage~B uses weighted child impurity as the
tested statistic and a left tail convention. The final minimum
impurity-decrease filter uses the corresponding weighted impurity decrease.
These quantities support split construction.

\subsection{Permutation budgets and early stopping}

Both stages use Phipson--Smyth $+1$ Monte Carlo tail estimates
\citep{PhipsonSmyth2010PermutationPValues}: right-tailed in Stage~A and
left-tailed in Stage~B. For the benchmarked CIT and CIF configurations, the
resampling budget is set from the Bonferroni-adjusted nodewise threshold and is
large enough for rejection to be possible on the discrete $+1$ scale.

\paragraph{Sequential stopping.}
The benchmarked CIT and CIF configurations use adaptive sequential stopping. The
stopping rule monitors permutation exceedances and stops once a Beta posterior
places the configured confidence on one side of the nodewise threshold. Because
the returned values are stopping-time estimates, the paper treats the fixed-node
guarantee and runtime ablations separately.

\subsection{Forests, resampling, and honesty}

CIF fits bootstrap trees and forms feature rankings by aggregating tree
split-importance vectors. The benchmark uses stratified bootstrap sampling for
classification and standard bootstrap sampling for regression. These forest-level
choices affect the fitted rankings, while the fixed-node theorem covers only the
Stage~A reference test.

\subsection{Benchmark coverage}

\Cref{sec:experiments-protocol,sec:experiments-repro} define the benchmark
protocol and configuration-selection rule. The benchmark compares 17
classification methods and 18 regression methods, including permutation-test
filters, embedded tree and ensemble methods, R ctree/cforest, and wrapper
methods. The rules below define which runs enter the appendix tables.

For per-method and direct pairwise summaries, each comparison uses the datasets,
downstream learners, and standard values of $k$ where the methods in that
comparison are present. Analyses that compare all methods restrict to datasets
with every method, downstream learner, and standard value of $k$ present. These
rules give 14 classification and 6 regression datasets for complete-case
analyses. Comparisons involving partykit ctree or cforest exclude
\nolinkurl{dexter}, because those R-method runs are absent; non-R comparisons
can still use \nolinkurl{dexter} when both methods are present.
\Cref{tab:coverage-skips} lists the classification seed skips and the R
ctree/cforest \texttt{dexter} exclusion.

\subsection{Randomness and reproducibility}

The benchmark configuration fixes the seed schedule used for all reported
runs. Pure Python permutation routines use local \texttt{numpy} generators.
Numba parallel permutation paths assign one seed per resample, which avoids
sharing a mutable generator across parallel loop iterations.

\begin{table}[H]
  \centering
  \footnotesize
  \setlength{\tabcolsep}{4pt}
  \begin{tabular}{@{}lll@{}}
    \toprule
    \tablehead{Method/config} & \tablehead{Dataset} & \tablehead{Seeds} \\
    \midrule
    R ctree \texttt{MonteCarlo} & \texttt{gisette} & \multicolumn{1}{c}{3} \\
    R ctree \texttt{MonteCarlo} & \texttt{isolet} & \multicolumn{1}{c}{2, 3} \\
    CIT with selector \texttt{rdc} & \texttt{gisette} & \multicolumn{1}{c}{0, 1, 3} \\
    CIT with selector \texttt{rdc} & \texttt{orlraws10P} & \multicolumn{1}{c}{1} \\
    partykit ctree, cforest & \texttt{dexter} & all seeds excluded for these R methods \\
    \bottomrule
  \end{tabular}
  \caption{Classification benchmark exclusions. The \texttt{dexter} row is a
  method-level exclusion for partykit ctree and cforest; the other rows are
  seed-level skips.}
  \label{tab:coverage-skips}
\end{table}

\section{Runtime and memory accounting}
\label{app:complexity-details}

The runtime terms behind the ablations in \cref{sec:results-practical} are
counted node by node: the scans, copies, and feature or threshold tests actually
reached. A node-by-node sum is needed because feature pools, node sizes, threshold
counts, and stopping times vary across the tree.

Let $n$ be the training sample count and let $p$ be the number of features. At
node $t$, let $n_t$ be the number of samples reaching the node. Using the
notation of \cref{sec:stageA-stageB}, write $a_t=|F_{t,\mathrm{avail}}|$ and
$d_t^{\mathrm{const}}=|F_{t,\mathrm{avail}}\setminus
F_{t,\mathrm{nonconst}}|$. After constant feature pruning and any
feature subsampling, Stage~A considers $F_t$, with
$m_t=|F_t|$. Let $d_t^{\mathrm{mute}}$ be the number of tested Stage~A features
removed from the descendant feature pool by feature muting;
if feature muting is disabled, $d_t^{\mathrm{mute}}=0$.

Stage~B is reached only when Stage~A returns a selected feature $j_t^\star$:
either the Stage~A permutation feature test rejects or selector permutation
testing is disabled. If Stage~B is reached, let $\tilde{h}_t$ be the number of
thresholds generated for $j_t^\star$ before any
minimum leaf size filtering. Let
$h_t=|C_{t,j_t^\star}|\le \tilde{h}_t$ be the number of valid thresholds that
remain after this filtering; if no filtering is applied, $h_t=\tilde{h}_t$.

Let $C^{\mathrm{sel}}_j(n_t)$ be the cost of the Stage~A statistic for feature
$j$ at node size $n_t$, excluding the label copies and shuffles counted
separately below. For a list of selectors, this denotes the cost of evaluating
the selector statistics for one observed or permuted draw and taking their
common-scale maximum. Let
$C^{\mathrm{split}}(n_t)$ be the cost of one Stage~B impurity statistic after
the threshold mask and child-label slices have been formed; the mask, slice,
label shuffle, and copy overhead is counted separately below. For the
multiple-correlation selector with Gini splitting in classification and the
Pearson-correlation selector with MSE splitting in regression, these statistic
evaluations are linear in $n_t$ for a fixed number of classes. The randomized
dependence coefficient (RDC) is $O(n_t\log n_t)$ for fixed random-projection and
class counts. For distance correlation and mutual information, the cost depends
on the backend used by $C^{\mathrm{sel}}_j(n_t)$: distance correlation may cost
$O(n_t\log n_t)$ or $O(n_t^2)$, and mutual information is delegated to
scikit-learn.

\paragraph{How runtime hyperparameters enter the accounting.}
\begin{itemize}
  \item \textbf{Feature subsampling.} This caps $m_t$ and therefore the
  Stage~A testing width. It does not cap the constant feature preprocessing pass
  over all $a_t$ available features. Feature-set construction and random
  ordering add at most linear overhead in the local available feature count.
  \item \textbf{Feature muting.} This can shrink $a_u$ and $m_u$ at descendant
  nodes $u$. It can also add $O(a_t d_t^{\mathrm{mute}})$ feature index updates
  at the current node when tested features whose Stage~A $p$-values do not
  reject are removed.
  \item \textbf{Threshold-search method and threshold cap.}
  Non-exact threshold methods can reduce the returned pre-filter threshold set
  $\tilde{h}_t$ when the configured cap is active. Exact search and the fallback
  for features with at most four unique values use all midpoints.
  \item \textbf{Adaptive stopping, feature scanning, and threshold scanning.}
  These reduce work only when they shorten the actual stopping path. They may
  change feature or threshold order and reduce the number of tested features,
  tested thresholds, or permutation draws, but the worst case is still all
  features or thresholds and the full budget. Bonferroni is separate because it
  changes the statistical rule.
  \item \textbf{Bonferroni correction and resampling presets.} When Bonferroni
  correction is enabled, the nodewise threshold is divided by the number of
  tests in that stage. Write $r_{\mathrm{tests}}$ for this count ($m_t$ in
  Stage~A and $h_t$ in Stage~B) and
  $\alpha_{\mathrm{stage}}\in\{\alpha_{\mathrm{sel}},\alpha_{\mathrm{split}}\}$
  for the corresponding nominal level. Named resampling presets compute the
  permutation budget at the stricter threshold
  $\alpha_{\mathrm{stage}}/r_{\mathrm{tests}}$: the \emph{minimum} budget scales
  like $r_{\mathrm{tests}}/\alpha_{\mathrm{stage}}$, the \emph{maximum} budget
  is $\lceil r_{\mathrm{tests}}^2/(4\alpha_{\mathrm{stage}}^2)\rceil$ up to
  rounding, and \emph{auto} uses the same stricter threshold in its
  normal-quantile formula. For an explicit integer budget, the budget is
  multiplied by $r_{\mathrm{tests}}$.
  \item \textbf{Bootstrap sample size, forest size, and worker count.} The
  bootstrap sample-size control affects the requested or capped bootstrap sample
  size, the forest size sets the number of tree fits, and the worker count
  affects parallel scheduling and memory pressure.
\end{itemize}

\paragraph{Per-node accounting.}
At each node, the tree builder first checks whether it must stop, including
whether all labels are equal, and computes the leaf value if it stops. This
costs $O(n_t)$ at every node. If the node can split, the builder scans all
available features and removes features that are constant at the node. This
costs $O(a_t n_t)$ for the data scan. Updating the array of available feature
indices costs another $O(a_t d_t^{\mathrm{const}})$ for constant feature
removals in the dense matrix algorithm, with worst case $O(a_t^2)$ if many
available features are constant at the same node. This pruning occurs before
feature subsampling, so subsampling caps Stage~A width but not the
initial constant feature scan.

Let $\mathcal{J}_t$ be the Stage~A features whose score or permutation test is
evaluated before any feature-level early stop. With Stage~A permutation tests,
this is the evaluated subset $E_t$ from \cref{sec:stageA-stageB}. It does not
include the optional feature-scanning prescan, which is counted separately. Let
$b^{\mathrm{sel}}_{t,j}$ be the number of permutation draws actually used for
feature $j$; when permutation tests are not run, take
$b^{\mathrm{sel}}_{t,j}=0$. The Stage~A testing work is
\begin{equation}
  S^{\mathrm{A}}_t
  =
  O\!\left(
    \sum_{j \in \mathcal{J}_t}
      \bigl(b^{\mathrm{sel}}_{t,j}+1\bigr)
      \bigl(C^{\mathrm{sel}}_j(n_t)+n_t\bigr)
  \right),
\end{equation}
where the $+1$ accounts for the observed statistic and the $n_t$ term accounts
for label copies, shuffles, and other per-draw linear overhead inside the
permutation test. Feature scanning, when enabled with selector early stopping,
computes prescan scores and sorts features, adding
\begin{equation}
  G^{\mathrm{A}}_t
  =
  O\!\left(
    \sum_{j \in F_t} C^{\mathrm{scan}}_j(n_t)
    + m_t \log m_t
  \right),
\end{equation}
where $C^{\mathrm{scan}}_j$ is the selector-score cost used for ordering. If
feature muting removes tested features whose Stage~A $p$-values do not reject,
updating the descendant available feature array adds
$O(a_t d_t^{\mathrm{mute}})$ work at the current node.

Stage~B first builds thresholds for the selected feature. An upper
bound is $O(n_t\log n_t)$, because unique values are computed and sorted before
threshold caps are applied. If the minimum leaf size is greater than one,
filtering adds $O(n_t\log n_t+\tilde{h}_t\log n_t)$ for the feature sort and
threshold lookups. Bounded non-exact methods reduce $\tilde{h}_t$ only when the
effective cap is below the number of available midpoints; with few unique values
or no effective cap, all available midpoints can remain in the threshold set.
They do not remove the initial unique-value pass.

Let $\mathcal{C}^{\mathrm{eval}}_t \subseteq C_{t,j_t^\star}$ be the retained
Stage~B thresholds whose permutation test or non-permutation split score is
evaluated before any threshold-level early stop. Let
$b^{\mathrm{split}}_{t,c}$ be the corresponding realized permutation count;
when splitter permutation tests are not run, take
$b^{\mathrm{split}}_{t,c}=0$. The Stage~B testing work is
\begin{equation}
  S^{\mathrm{B}}_t
  =
  O\!\left(
    \sum_{c \in \mathcal{C}^{\mathrm{eval}}_t}
      \bigl(b^{\mathrm{split}}_{t,c}+1\bigr)
      \bigl(C^{\mathrm{split}}(n_t)+n_t\bigr)
  \right).
\end{equation}
In $S^{\mathrm{B}}_t$, the added $n_t$ term accounts for forming or applying
threshold masks, child-label slices, label copies, and shuffles for each
observed or permuted draw.
Under splitter early stopping, threshold ordering has two paths. If threshold
scanning is enabled, it uses weighted child impurity as an ordering score and
adds $O(h_t(C^{\mathrm{split}}(n_t)+n_t) + h_t\log h_t)$. If threshold scanning
is disabled, the retained thresholds are still randomly permuted, adding
$O(h_t)$. These ordering steps are separate from the
permutation $p$-value computation.

Finally, if Stage~B accepts a split and the impurity decrease is large enough,
the dense matrix tree builder creates left and right child matrices with all $p$
columns. This copy costs $O(n_t p)$. Nodes that reach Stage~A but become leaves
because Stage~A or Stage~B does not reject still incur the work already
described. Terminal leaves can also occur when constant feature pruning leaves no
available feature, threshold generation or leaf size filtering leaves no valid
threshold, or an accepted threshold fails the minimum impurity-decrease gate.
These leaves contribute the costs incurred up to that point but do not allocate
child matrices.

\paragraph{Fixed-$B$ and adaptive modes.}
In the exhaustive fixed-$B$ reference setting, Stage~A tests every feature in
$F_t$, Stage~B tests every retained threshold, and both stages use fixed
permutation budgets. In adaptive runs, stopping can reduce work by testing fewer
features or thresholds or using fewer permutation draws. Feature and threshold
scanning are ordering heuristics: they pay for a prescan so likely features or
thresholds are tried earlier. They can reduce realized runtime when the shorter
stopping path offsets the prescan work, but they do not improve the worst-case
bound because the stopping rule can still examine all features or thresholds and
use the full budget.

\paragraph{From nodes to trees and forests.}
Let $W_t$ be the cost paid at node $t$: the stopping checks, any Stage~A work,
any Stage~B work, and any child-copy work reached by the builder. Tree training
cost is the sum over all realized tree nodes, not only over internal split
nodes:
\begin{equation}
  T_{\mathrm{tree}} = \sum_{t \in \mathcal{V}(\mathcal{T})} W_t.
\end{equation}
Here $\mathcal{V}(\mathcal{T})$ and $\mathcal{I}(\mathcal{T})$ are the realized
node set and internal split-node set. Leaves stopped by the initial checks
contribute only the stopping and value computation term at that node, while leaves
that fail after Stage~A or Stage~B contribute the corresponding testing terms. A
binary tree has
$|\mathcal{V}(\mathcal{T})|=2|\mathcal{I}(\mathcal{T})|+1$. A depth cap gives a
coarse exponential node count bound, but $n_t$, $a_t$, $m_t$, $\tilde{h}_t$,
$h_t$, and the realized permutation counts usually change with depth.

If honesty is enabled, the node sum is computed on the structure building
subsample. Training also pays $O(np)$ for the dense split/estimation partition
and a leaf re-estimation pass. If $e$ is the estimation-sample size and
$\ell_i$ is the fitted-tree path length for estimation sample $i$, the current
tuple-based path representation gives worst-case routing cost
$O(\sum_{i=1}^{e}\ell_i^2)$, followed by a stored-path tree walk of
$O(\sum_{v\in\mathcal{V}(\mathcal{T})} d_v + e)$, where $d_v$ is node depth.

For a forest of $M$ trees, let $s_r$ be the number of samples used to fit tree
$r$ after bootstrap and classifier sampling. Let $C^{\mathrm{sample}}_r$ be the
cost of constructing those sample indices. Row selection and per-tree dense
validation/copy add $O(s_r p)$ before tree growth. The forest fit first
validates the full input and then fits the trees.
Write $T_{\mathrm{tree},r}$ for the node-sum cost of tree $r$ computed as above.
The total computational work of the forest fit, before parallel wall-clock
scheduling effects, can be summarized as
\begin{equation}
  T_{\mathrm{forest}}
  =
  O(np)
  + \sum_{r=1}^{M}
      \left(C^{\mathrm{sample}}_r + O(s_r p) + T_{\mathrm{tree},r}
      + \mathbf{1}_{\mathrm{honest}} H_{\mathrm{honest},r}\right)
  + O(Mp)
  + \mathbf{1}_{\mathrm{oob}} T_{\mathrm{oob}},
\end{equation}
where $\mathbf{1}_{\mathrm{honest}}$ is one when honesty is enabled and zero
otherwise, $H_{\mathrm{honest},r}$ denotes the honesty partition and leaf
re-estimation cost just described for tree $r$, $\mathbf{1}_{\mathrm{oob}}$ is
one when OOB scoring is requested and zero otherwise, and $T_{\mathrm{oob}}$ is
the optional post-fit OOB scoring cost, separate from the tree-building sum.
With joblib parallelism, wall-clock time depends on the effective worker count,
load balance across trees, process overhead, and memory bandwidth; the
expression above is a work accounting.

Without bootstrap sampling, each tree starts with root size $n$ and the
bootstrap cap is unused. With bootstrap sampling, the configured bootstrap
sample-size parameter sets the requested sample size or an upper bound, and
classifier sampling can reduce $s_r$ further. If honesty is enabled, the root
size for structure building is the splitting subsample size.

\paragraph{Memory accounting.}
At a node, feature index arrays and scanning arrays use $O(a_t+m_t)$ space,
and threshold construction uses $O(n_t+\tilde{h}_t+h_t)$ auxiliary space.
Permutation tests add workspaces for statistics, masks, labels, permutation
outputs, and selector-specific arrays. Fixed-budget non-adaptive tests may store
an $O(B_{\mathrm{stage}})$ vector of permuted statistics, where
$B_{\mathrm{stage}}$ is the configured fixed budget for the selector or splitter
test being run. Adaptive paths can keep constant-size counters apart from copied
or shuffled labels and selector-specific arrays.

The dense matrix terms dominate in many fits. Validation materializes a dense
$O(np)$ feature matrix. Each accepted split materializes dense left and right
child matrices with all $p$ columns and combined size $O(n_t p)$. During
depth-first recursion, parent frames retain sibling child arrays while one child
is being grown. If $\mathcal{P}$ is the set of accepted split nodes whose child
matrices are live on the active recursion stack, the dense-array part of peak
single-tree training memory is bounded by
\[
  O\!\left(np + p\sum_{u\in\mathcal{P}} n_u\right),
\]
plus node-local auxiliary workspaces. This is a live-path bound: it is $O(np)$
for balanced growth but can be larger for highly unbalanced recursion.

Honesty adds memory during fitting for the dense split and estimation partitions.
It also stores estimation-sample groups by leaf path, giving
$O(e+\sum_{q\in\mathcal{Q}} |q|)$ path-index workspace for the set
$\mathcal{Q}$ of retained leaf-path keys, bounded by $O(eD)$ when all estimation
samples have path length at most $D$. This is in addition to the tree-building
memory above.

Forest peak training memory grows with concurrent tree fits. Let
$w=\min(M,n_{\mathrm{jobs}})$, where $n_{\mathrm{jobs}}$ is the number of
parallel jobs used by the joblib backend after parameter normalization. Peak
memory includes the validated forest input, returned estimators, and up to $w$
concurrently executing per-tree workspaces: bootstrap indices, dense
row-selected matrices, validation copies, and peak memory terms for each tree.
Depending on joblib serialization or memmap behavior, workers may also hold
shared memmaps or copies of the validated input.

For the same $M$-tree forest, fitted storage includes stored node records for
all trees, the estimator list, $O(Mp)$ per-tree arrays such as feature names and
feature importances, and $O(p)$ forest-level metadata arrays. If OOB scoring is
enabled, fitted storage also includes $O(nK)$ classifier OOB decision values for
$K$ classes or $O(n)$ regressor OOB predictions, plus the scalar OOB score. OOB
scoring also uses temporary workspace during fitting for the OOB count vector,
recomputed bootstrap and unsampled indices, prediction slices
$X_{\mathrm{oob}}$, and per-tree predictions or class-probability arrays.

\FloatBarrier
\section{Datasets}
\label{app:dataset-inventory}

\subsection{Benchmark datasets}

The dataset inventory covers the real-data benchmark and the synthetic
feature-recovery analyses.
For each real-data benchmark dataset, the tables report the number of samples
$n$, the number of features $p$, and the source collection. Source labels refer
to UCI Machine Learning Repository \citep{KellyLongjohnNottinghamUCI}, the NIPS
2003 feature selection challenge \citep{Guyon2004NIPSFeatureSelectionChallenge},
ASU/scikit-feature \citep{Li2018FeatureSelectionDataPerspective}, OpenML
\citep{Vanschoren2014OpenML}, and CoEPrA molecular-property benchmarks
\citep{DemirKavuk2010CoeprA}. ASU/SF abbreviates ASU/scikit-feature.
Slash-separated labels indicate original benchmark provenance and OpenML
availability or hosting.

\begin{table}[H]
  \centering
  \scriptsize
  \setlength{\tabcolsep}{5pt}
  \renewcommand{\arraystretch}{1.0}
  \begin{tabular}{@{}lS[table-format=5.0]S[table-format=5.0]l@{}}
    \toprule
    \tablehead{Dataset} & {$n$} & {$p$} & \tablehead{Source} \\
    \midrule
    \nolinkurl{gamma} & 19,020 & 9 & UCI \\
    \nolinkurl{glass} & 214 & 10 & UCI \\
    \nolinkurl{page-blocks} & 5,473 & 10 & UCI \\
    \nolinkurl{vowel-context} & 990 & 10 & UCI \\
    \nolinkurl{wine} & 178 & 13 & UCI \\
    \nolinkurl{letter} & 20,000 & 16 & UCI \\
    \nolinkurl{pendigits} & 10,992 & 16 & UCI \\
    \nolinkurl{spam} & 4,601 & 57 & UCI \\
    \nolinkurl{musk} & 6,597 & 166 & UCI \\
    \nolinkurl{madelon} & 2,000 & 500 & NIPS 2003 \\
    \nolinkurl{isolet} & 7,797 & 616 & UCI \\
    \nolinkurl{ORL} & 400 & 1,024 & ASU/SF \\
    \nolinkurl{Yale} & 165 & 1,024 & ASU/SF \\
    \nolinkurl{warpAR10P} & 130 & 2,400 & ASU/SF \\
    \nolinkurl{warpPIE10P} & 210 & 2,420 & ASU/SF \\
    \nolinkurl{gisette} & 6,000 & 5,000 & NIPS 2003 \\
    \nolinkurl{TOX_171} & 171 & 5,748 & ASU/SF \\
    \nolinkurl{ALLAML} & 72 & 7,129 & ASU/SF \\
    \nolinkurl{arcene} & 100 & 10,000 & NIPS 2003 \\
    \nolinkurl{pixraw10P} & 100 & 10,000 & ASU/SF \\
    \nolinkurl{orlraws10P} & 100 & 10,304 & ASU/SF \\
    \nolinkurl{CLL_SUB_111} & 111 & 11,340 & ASU/SF \\
    \nolinkurl{dexter} & 300 & 20,000 & NIPS 2003 \\
    \bottomrule
  \end{tabular}
  \caption{Classification benchmark datasets.}
  \label{tab:benchmark-dataset-inventory-clf}
\end{table}

\begin{table}[H]
  \centering
  \scriptsize
  \setlength{\tabcolsep}{5pt}
  \begin{tabular}{@{}lS[table-format=4.0]S[table-format=4.0]l@{}}
    \toprule
    \tablehead{Dataset} & {$n$} & {$p$} & \tablehead{Source} \\
    \midrule
    \nolinkurl{facebook} & 500 & 21 & UCI/OpenML \\
    \nolinkurl{imports-85} & 205 & 76 & UCI \\
    \nolinkurl{residential} & 372 & 107 & UCI/OpenML \\
    \nolinkurl{community_crime} & 1,994 & 1,954 & UCI/OpenML \\
    \nolinkurl{comm_violence} & 2,215 & 2,210 & UCI/OpenML \\
    \nolinkurl{coepra2} & 76 & 5,144 & CoEPrA/OpenML \\
    \nolinkurl{coepra1} & 89 & 5,787 & CoEPrA/OpenML \\
    \nolinkurl{coepra3} & 133 & 5,787 & CoEPrA/OpenML \\
    \bottomrule
  \end{tabular}
  \caption{Regression benchmark datasets.}
  \label{tab:benchmark-dataset-inventory-reg}
\end{table}

\FloatBarrier

\subsection{Synthetic datasets}

These tables report $n$, $p$, and key settings for each
synthetic feature-recovery design. In the classification table, \texttt{sep}
denotes class separation and \texttt{flip} denotes label-flip probability. In
the regression table, \texttt{noise} denotes the response-noise scale; for the
heteroscedastic row, \texttt{scale} controls the feature-dependent noise
multiplier.

The synthetic datasets are controlled recovery problems with known informative
features. The classification designs cover redundant features,
high-cardinality noise, nonlinear signal, standard and weak signal, Toeplitz
correlation, correlated-noise features, and sparse high-$p$ structure. The
regression designs cover heteroscedastic noise, redundant features, Friedman~1,
standard and weak signal, Toeplitz correlation, correlated-noise features, and
sparse high-$p$ structure. The settings keep the examples compact while
including one high-dimensional, small-sample case for each task.

\begin{table}[H]
  \centering
  \scriptsize
  \setlength{\tabcolsep}{4pt}
  \begin{tabular}{@{}L{0.26\linewidth}L{0.32\linewidth}S[table-format=4.0]S[table-format=4.0]@{}}
    \toprule
    \tablehead{Design} & \tablehead{Setting} & {$n$} & {$p$} \\
    \midrule
    Redundant features & 10 informative, 20 redundant features & 1,000 & 50 \\
    High-cardinality noise & 10 informative, 50 high-cardinality noise features, 500 levels & 1,000 & 100 \\
    Nonlinear signal & binarized Friedman~1, 5 informative & 1,000 & 100 \\
    Standard classification & 10 informative, sep $= 2.0$ & 1,000 & 100 \\
    Toeplitz correlation & 10 informative, $\rho = 0.95$ & 1,000 & 100 \\
    Weak signal classification & 10 informative, sep $= 0.1$, flip $= 0.1$ & 1,000 & 100 \\
    Correlated noise & 10 informative, 20 added correlated-noise features, $\rho = 0.9$ & 1,000 & 120 \\
    Sparse high-$p$ classification & 5 informative, sep $= 0.5$ & 200 & 1,000 \\
    \bottomrule
  \end{tabular}
  \caption{Synthetic classification designs used in the feature-recovery
  analysis.}
  \label{tab:synthetic-dataset-inventory-clf}
\end{table}

\begin{table}[H]
  \centering
  \scriptsize
  \setlength{\tabcolsep}{4pt}
  \begin{tabular}{@{}L{0.26\linewidth}L{0.32\linewidth}S[table-format=4.0]S[table-format=3.0]@{}}
    \toprule
    \tablehead{Design} & \tablehead{Setting} & {$n$} & {$p$} \\
    \midrule
    Heteroscedastic regression & heteroscedastic scale $= 4.0$, noise $= 5.0$ & 1,000 & 50 \\
    Redundant features & 10 informative, 20 redundant features & 1,000 & 50 \\
    Friedman 1 & 5 informative, noise $= 1.0$ & 1,000 & 100 \\
    Standard regression & 10 informative, noise $= 1.0$ & 1,000 & 100 \\
    Toeplitz correlation & 10 informative, $\rho = 0.95$ & 1,000 & 100 \\
    Weak signal regression & 10 informative, noise $= 50.0$ & 1,000 & 100 \\
    Correlated noise & 10 informative, 20 added correlated-noise features, $\rho = 0.9$ & 1,000 & 120 \\
    Sparse high-$p$ regression & 5 informative, noise $= 10.0$ & 200 & 500 \\
    \bottomrule
  \end{tabular}
  \caption{Synthetic regression designs used in the feature-recovery analysis.}
  \label{tab:synthetic-dataset-inventory-reg}
\end{table}

\section{Benchmark stability}
\label{app:benchmark-stability}

The benchmark-stability checks report dataset counts, complete-case CIF
comparisons, leave-one-dataset-out (LODO) configuration sensitivity, and
sensitivity checks by downstream learner, standard values of $k$, and seed.
\Cref{tab:benchmark-dataset-counts} summarizes the dataset counts used by the
benchmark analyses. Counts differ because each analysis is restricted to the
datasets, downstream learners, and standard top-$k$ values available for that
comparison. The main real-data counts are unique datasets with at least one
complete learner-by-$k$ combination for all methods; the complete-case analyses
below use the stricter set with every standard learner-by-$k$ combination
present.

\begin{table}[H]
  \centering
  \fontsize{8}{9.1}\selectfont
  \setlength{\tabcolsep}{4pt}
  \begin{tabularx}{0.88\linewidth}{@{}L{0.18\linewidth} C{0.18\linewidth} C{0.15\linewidth} Y@{}}
    \toprule
    \tablehead{Analysis} &
    Classification datasets & Regression datasets &
    \multicolumn{1}{>{\centering\arraybackslash}X}{Purpose} \\
    \midrule
    Main real-data comparison & 22 & 8 & Primary rank summary on real datasets \\
    Complete-case CIF and rank checks & 14 & 6 & Bootstrap intervals and omnibus rank summaries \\
    LODO configuration sensitivity & 14 & 6 & Held-out-dataset configuration checks \\
    Learner $\times$ $k$ sensitivity & 22, 21, 15, 15, 14 & 8, 8, 7, 7, 6 & Method ordering by downstream learner and value of $k$ \\
    Seed sensitivity & 12 & 6 & Seed variation on datasets complete across seeds \\
    \bottomrule
  \end{tabularx}
  \caption{Dataset counts for benchmark analyses. The learner-by-$k$ row lists
  counts at $k=5,10,25,50,100$; each summary uses the datasets, downstream
  learners, and standard top-$k$ values available for its comparison.}
  \label{tab:benchmark-dataset-counts}
\end{table}

\subsection{Bootstrap and rank checks}

\Cref{tab:cif-ci-comparisons} gives complete-case CIF comparisons against
ctree, cforest, and CIT after restricting to datasets with all methods,
downstream learners, and all five standard top-$k$ values available. Within
each task, all comparisons use the same datasets, the same downstream learners,
and $k \in \{5,10,25,50,100\}$.

\begin{table}[H]
  \centering
  \footnotesize
  \setlength{\tabcolsep}{4pt}
  \begin{tabular}{@{}ll S[table-format=1.3] c c@{}}
    \toprule
    \tablehead{Task} & \tablehead{Compared method} & {Mean $\Delta$} & 95\% CI & W--L \\
    \midrule
    Classification & ctree & 0.109 & [0.068, 0.154] & \winloss{14}{0} \\
    Classification & cforest & 0.111 & [0.067, 0.160] & \winloss{13}{1} \\
    Classification & CIT & 0.041 & [0.023, 0.060] & \winloss{13}{1} \\
    Regression & ctree & 0.297 & [0.113, 0.529] & \winloss{6}{0} \\
    Regression & cforest & 0.957 & [0.040, 2.570] & \winloss{5}{1} \\
    Regression & CIT & 0.480 & [-0.007, 1.224] & \winloss{5}{1} \\
    \bottomrule
  \end{tabular}
  \caption{Complete-case CIF differences against ctree, cforest, and CIT. Mean
  $\Delta$ is CIF minus comparator after averaging each dataset over the same
  downstream learners and values of $k$. Intervals resample dataset-level paired
  deltas with 20{,}000 percentile-bootstrap replicates. W--L gives dataset-level
  wins and losses; no exact ties occurred.}
  \label{tab:cif-ci-comparisons}
\end{table}

The metric is balanced accuracy for classification and $R^2$ for regression.
Positive differences favor CIF. The intervals summarize dataset-level paired
deltas after configuration selection. Large positive regression differences can
reflect negative mean $R^2$ for the comparator.

The Friedman checks use dataset-level method means after averaging the 15
learner-by-$k$ combinations. Classification: $\chi^2(16)=141.32$, $p<0.001$,
Kendall's $W=0.63$. Regression: $\chi^2(17)=33.79$, $p=0.0089$, Kendall's
$W=0.33$. For the 6-dataset regression panel, read the $p$-value together with
Kendall's $W$ and the LODO table.

\subsection{LODO configuration sensitivity}

\Cref{tab:lodo-config-sensitivity} reports LODO configuration selection for each
task. Each run selects configurations on all but one complete-case dataset in
that task, evaluates the held-out dataset, and averages scores and ranks over
downstream learners and standard top-$k$ values. CIF reselects the
main-benchmark configuration in all 14 classification holdouts. In regression, it uses
two configurations, reselects the global regression configuration in 5 of 6 holdouts, and
ties CatBoost for second by mean LODO rank behind ExtraTrees.

\begin{table}[!htbp]
  \centering
  \fontsize{8}{9.1}\selectfont
  \setlength{\tabcolsep}{3pt}
  \renewcommand{\arraystretch}{1.0}
  \begin{tabular*}{0.96\linewidth}{@{\extracolsep{\fill}}l
    S[table-format=2.2]
    S[table-format=1.3]
    S[table-format=1.0]
    c
    @{\hspace{0.9em}}
    l
    S[table-format=2.2]
    S[table-format=-1.3]
    S[table-format=1.0]
    c@{}}
    \toprule
    \multicolumn{5}{c}{Classification} &
    \multicolumn{5}{c}{Regression} \\
    \cmidrule(lr){1-5}\cmidrule(l){6-10}
    \tablehead{Method} & {Rank} & {Score} & {Settings} & {Main} &
    \tablehead{Method} & {Rank} & {Score} & {Settings} & {Main} \\
    \midrule
    LightGBM & 3.14 & 0.816 & 1 & \countfrac{14}{14} & ExtraTrees & 4.17 & 0.292 & 1 & \countfrac{6}{6} \\
    CatBoost & 4.29 & 0.805 & 1 & \countfrac{14}{14} & CatBoost & 5.50 & 0.229 & 1 & \countfrac{6}{6} \\
    RF & 4.79 & 0.802 & 1 & \countfrac{14}{14} & CIF & 5.50 & 0.214 & 2 & \countfrac{5}{6} \\
    XGBoost & 4.79 & 0.809 & 2 & \countfrac{6}{14} & RF & 6.00 & 0.260 & 1 & \countfrac{6}{6} \\
    CIF & 5.64 & 0.798 & 1 & \countfrac{14}{14} & RT & 6.00 & 0.266 & 1 & \countfrac{6}{6} \\
    ExtraTrees & 5.79 & 0.789 & 1 & \countfrac{14}{14} & DT & 8.00 & 0.124 & 1 & \countfrac{6}{6} \\
    RF-RFE & 6.57 & 0.795 & 1 & \countfrac{14}{14} & XGBoost & 8.50 & 0.149 & 2 & \countfrac{5}{6} \\
    DT & 7.43 & 0.774 & 1 & \countfrac{14}{14} & LightGBM & 9.67 & -0.214 & 1 & \countfrac{6}{6} \\
    Boruta & 9.14 & 0.747 & 1 & \countfrac{14}{14} & PI & 9.83 & -0.033 & 1 & \countfrac{6}{6} \\
    RT & 10.64 & 0.752 & 1 & \countfrac{14}{14} & PC filter & 9.83 & 0.098 & 1 & \countfrac{6}{6} \\
    CIT & 10.79 & 0.757 & 1 & \countfrac{14}{14} & RDC filter & 10.00 & -0.038 & 1 & \countfrac{6}{6} \\
    MC filter & 11.43 & 0.734 & 1 & \countfrac{14}{14} & CIT & 11.33 & -0.291 & 2 & \countfrac{5}{6} \\
    RDC filter & 11.86 & 0.707 & 1 & \countfrac{14}{14} & DC filter & 11.50 & -0.009 & 1 & \countfrac{6}{6} \\
    cforest & 13.29 & 0.687 & 1 & \countfrac{14}{14} & Boruta & 11.83 & 0.109 & 1 & \countfrac{6}{6} \\
    ctree & 13.64 & 0.689 & 1 & \countfrac{14}{14} & RF-RFE & 11.83 & 0.109 & 1 & \countfrac{6}{6} \\
    PI & 13.75 & 0.689 & 1 & \countfrac{14}{14} & CPI & 13.50 & -0.552 & 1 & \countfrac{6}{6} \\
    CPI & 16.04 & 0.598 & 1 & \countfrac{14}{14} & ctree & 13.83 & -0.035 & 1 & \countfrac{6}{6} \\
    \multicolumn{5}{c}{} & cforest & 14.17 & -0.925 & 2 & \countfrac{5}{6} \\
    \bottomrule
  \end{tabular*}
  \caption{LODO configuration sensitivity. Lower rank is better; rows are
  ordered by mean LODO rank within each task. Score is the mean held-out metric.
  Settings reports the number of distinct selected configurations; Main reports
  how often the main benchmark configuration is reselected. Ranks are computed
  within held-out datasets
  before averaging, so score order need not match rank order.}
  \label{tab:lodo-config-sensitivity}
\end{table}

\subsection{Breakdowns by downstream learner, \texorpdfstring{top-$k$}{top-k} value, and seed}

The main benchmark averages over downstream learners, standard values of $k$,
folds, and seeds. The next tables show how method ordering changes when results
are split by downstream learner, value of $k$, and seed. They use the
same selected method configurations as the main benchmark.

For the learner and $k$ checks, one unit is one downstream learner paired with
one standard value of $k$. Within each unit, methods are ordered by mean rank
among datasets where every benchmark method is available for that learner-by-$k$
combination; ties use average ranks.
\Cref{tab:classification-learner-k-sensitivity,tab:regression-learner-k-sensitivity}
show all 15 learner-by-$k$ combinations for each task. Rows follow the main
benchmark order. The
learner columns report mean position over the five standard values of $k$ for
that learner; lower positions are better. Mean position averages the positions
obtained after ordering methods within each learner-by-$k$ combination; mean
rank averages the dataset-level ranks directly. Top 3 of 15 and Top 5 of 15
count the learner-by-$k$ combinations where the method ranks in the top 3 or
top 5.

\begin{table}[H]
  \centering
  \footnotesize
  \setlength{\tabcolsep}{3pt}
  \renewcommand{\arraystretch}{1.04}
  \begin{tabular}{@{}l
    S[table-format=2.1]
    S[table-format=2.1]
    S[table-format=2.1]
    S[table-format=2.1]
    S[table-format=2.1]
    c
    c
    c@{}}
    \toprule
    \tablehead{Method} & {KNN} & {LR} & {SVM} & {Mean position} & {Mean rank} & {Best--worst} &
    \shortstack{Top 3\\of 15} & \shortstack{Top 5\\of 15} \\
    \midrule
    LightGBM & 3.0 & 1.6 & 1.6 & 2.1 & 4.5 & 1.0--7.0 & \countfrac{14}{15} & \countfrac{14}{15} \\
    XGBoost & 4.0 & 2.2 & 2.2 & 2.8 & 5.0 & 1.0--5.0 & \countfrac{12}{15} & \countfrac{15}{15} \\
    CatBoost & 2.9 & 3.4 & 3.4 & 3.2 & 5.2 & 1.0--7.0 & \countfrac{11}{15} & \countfrac{12}{15} \\
    CIF & 3.0 & 5.6 & 4.6 & 4.4 & 6.0 & 1.0--8.0 & \countfrac{3}{15} & \countfrac{12}{15} \\
    RF & 4.8 & 4.6 & 4.8 & 4.7 & 6.2 & 1.0--7.0 & \countfrac{1}{15} & \countfrac{11}{15} \\
    RF-RFE & 5.5 & 5.8 & 5.6 & 5.6 & 6.9 & 1.5--8.0 & \countfrac{2}{15} & \countfrac{5}{15} \\
    ExtraTrees & 6.4 & 6.2 & 8.0 & 6.9 & 7.2 & 4.0--9.0 & \countfrac{0}{15} & \countfrac{2}{15} \\
    DT & 8.1 & 6.9 & 7.0 & 7.3 & 8.0 & 3.0--12.0 & \countfrac{2}{15} & \countfrac{4}{15} \\
    RT & 10.6 & 9.8 & 10.0 & 10.1 & 9.9 & 8.0--12.0 & \countfrac{0}{15} & \countfrac{0}{15} \\
    CIT & 11.2 & 9.8 & 10.2 & 10.4 & 10.1 & 8.0--15.0 & \countfrac{0}{15} & \countfrac{0}{15} \\
    Boruta & 9.6 & 11.5 & 9.8 & 10.3 & 9.5 & 6.0--14.0 & \countfrac{0}{15} & \countfrac{0}{15} \\
    MC filter & 10.7 & 11.8 & 12.0 & 11.5 & 10.8 & 9.0--13.0 & \countfrac{0}{15} & \countfrac{0}{15} \\
    PI & 13.4 & 15.0 & 13.2 & 13.9 & 12.1 & 10.0--16.0 & \countfrac{0}{15} & \countfrac{0}{15} \\
    cforest & 14.2 & 12.4 & 14.6 & 13.7 & 12.1 & 11.0--16.0 & \countfrac{0}{15} & \countfrac{0}{15} \\
    RDC filter & 13.4 & 14.1 & 13.6 & 13.7 & 12.1 & 10.0--16.0 & \countfrac{0}{15} & \countfrac{0}{15} \\
    ctree & 15.6 & 15.5 & 15.4 & 15.5 & 13.0 & 13.0--17.0 & \countfrac{0}{15} & \countfrac{0}{15} \\
    CPI & 16.6 & 16.8 & 17.0 & 16.8 & 14.4 & 15.0--17.0 & \countfrac{0}{15} & \countfrac{0}{15} \\
    \bottomrule
  \end{tabular}
  \caption{Classification sensitivity by downstream learner and value of $k$.
  Lower values are better. KNN, LR, and SVM denote downstream learners;
  Top 3 of 15 and Top 5 of 15 count the learner-by-$k$ combinations where the
  method ranks in the top 3 or top 5. Mean position averages ordinal method
  positions over the 15 combinations; mean rank averages the underlying
  dataset-level ranks. The dataset counts at
  $k=5,10,25,50,100$ are
  22, 21, 15, 15, and 14.}
  \label{tab:classification-learner-k-sensitivity}
\end{table}

In the classification learner-by-$k$ sensitivity table, CIF is top five in 12 of
the 15 learner-by-$k$ combinations and top three in 3 of 15. Its mean position
is close to RF's and is better than the single-tree and filter baselines in the
mean-position summary. LightGBM, XGBoost, and CatBoost have lower mean
positions.

\begin{table}[H]
  \centering
  \footnotesize
  \setlength{\tabcolsep}{3pt}
  \renewcommand{\arraystretch}{1.04}
  \begin{tabular}{@{}l
    S[table-format=2.1]
    S[table-format=2.1]
    S[table-format=2.1]
    S[table-format=2.1]
    S[table-format=2.1]
    c
    c
    c@{}}
    \toprule
    \tablehead{Method} & {KNN} & {Ridge} & {SVR} & {Mean position} & {Mean rank} & {Best--worst} &
    \shortstack{Top 3\\of 15} & \shortstack{Top 5\\of 15} \\
    \midrule
    ExtraTrees & 5.1 & 2.7 & 2.1 & 3.3 & 6.3 & 1.0--9.0 & \countfrac{10}{15} & \countfrac{12}{15} \\
    CatBoost & 1.4 & 7.1 & 3.0 & 3.8 & 6.9 & 1.0--12.5 & \countfrac{8}{15} & \countfrac{12}{15} \\
    CIF & 3.2 & 7.2 & 3.4 & 4.6 & 7.1 & 1.0--10.0 & \countfrac{7}{15} & \countfrac{10}{15} \\
    RF & 6.7 & 3.0 & 8.4 & 6.0 & 7.9 & 1.0--11.5 & \countfrac{4}{15} & \countfrac{6}{15} \\
    RT & 4.2 & 8.2 & 5.6 & 6.0 & 8.0 & 1.5--13.0 & \countfrac{3}{15} & \countfrac{10}{15} \\
    DT & 5.9 & 8.4 & 7.3 & 7.2 & 8.5 & 1.0--14.0 & \countfrac{2}{15} & \countfrac{2}{15} \\
    LightGBM & 7.5 & 9.0 & 8.7 & 8.4 & 8.9 & 1.0--16.5 & \countfrac{3}{15} & \countfrac{4}{15} \\
    RDC filter & 9.8 & 12.2 & 5.3 & 9.1 & 9.3 & 3.0--17.5 & \countfrac{1}{15} & \countfrac{3}{15} \\
    CIT & 12.4 & 4.1 & 12.1 & 9.5 & 9.4 & 2.0--16.0 & \countfrac{2}{15} & \countfrac{3}{15} \\
    PC filter & 10.9 & 11.5 & 9.3 & 10.6 & 9.8 & 7.5--16.0 & \countfrac{0}{15} & \countfrac{0}{15} \\
    DC filter & 13.7 & 14.0 & 7.4 & 11.7 & 10.3 & 4.0--17.0 & \countfrac{0}{15} & \countfrac{1}{15} \\
    XGBoost & 6.0 & 14.8 & 9.1 & 10.0 & 9.5 & 2.0--16.0 & \countfrac{1}{15} & \countfrac{4}{15} \\
    Boruta & 12.4 & 7.1 & 12.6 & 10.7 & 9.8 & 2.0--16.0 & \countfrac{1}{15} & \countfrac{2}{15} \\
    PI & 10.2 & 6.2 & 13.8 & 10.1 & 9.7 & 2.0--15.0 & \countfrac{1}{15} & \countfrac{2}{15} \\
    RF-RFE & 14.8 & 11.0 & 12.4 & 12.7 & 10.8 & 3.0--18.0 & \countfrac{1}{15} & \countfrac{1}{15} \\
    cforest & 17.3 & 12.7 & 16.1 & 15.4 & 12.5 & 5.5--18.0 & \countfrac{0}{15} & \countfrac{0}{15} \\
    ctree & 12.6 & 16.0 & 17.8 & 15.5 & 12.9 & 8.0--18.0 & \countfrac{0}{15} & \countfrac{0}{15} \\
    CPI & 16.9 & 15.8 & 16.6 & 16.4 & 13.2 & 9.0--18.0 & \countfrac{0}{15} & \countfrac{0}{15} \\
    \bottomrule
  \end{tabular}
  \caption{Regression sensitivity by downstream learner and value of $k$.
  Lower values are better. KNN, Ridge, and SVR denote downstream
  learners; Top 3 of 15 and Top 5 of 15 count the learner-by-$k$ combinations
  where the method ranks in the top 3 or top 5. Mean position averages ordinal
  method positions over the 15 combinations; mean rank averages the underlying
  dataset-level ranks. The dataset counts at
  $k=5,10,25,50,100$ are
  8, 8, 7, 7, and 6.}
  \label{tab:regression-learner-k-sensitivity}
\end{table}

CIF has the third-lowest regression mean position in this breakdown, after
ExtraTrees and CatBoost, and remains ahead of ctree, cforest, CIT, and most
filter and wrapper methods in this summary. The pairwise figures use the
datasets available for each comparator; the complete-case panels above use the
14/6 all-method dataset sets.

\begin{figure}[H]
  \centering
  \includegraphics[width=0.88\linewidth]{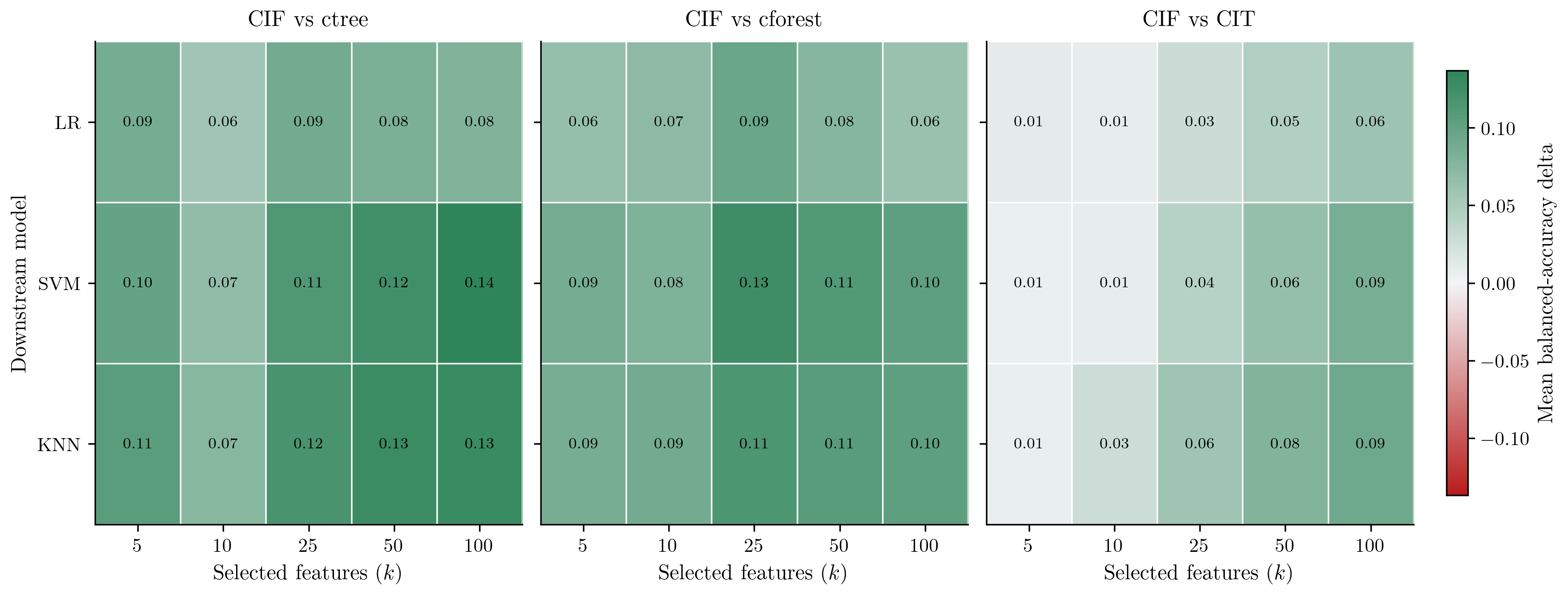}
  \caption{Classification CIF margins against ctree, cforest, and CIT by
  downstream learner and value of $k$. Each plotted value is the
  mean balanced-accuracy difference, CIF minus comparator, for one
  downstream learner and one value of $k$; positive values favor CIF. Dataset
  counts by $k$ are 22, 21, 15, 15, and 14 for ctree and cforest, and
  23, 22, 16, 16, and 15 for CIT.}
  \label{fig:benchmark-pairwise-sensitivity}
\end{figure}

\begin{figure}[H]
  \centering
  \includegraphics[width=0.88\linewidth]{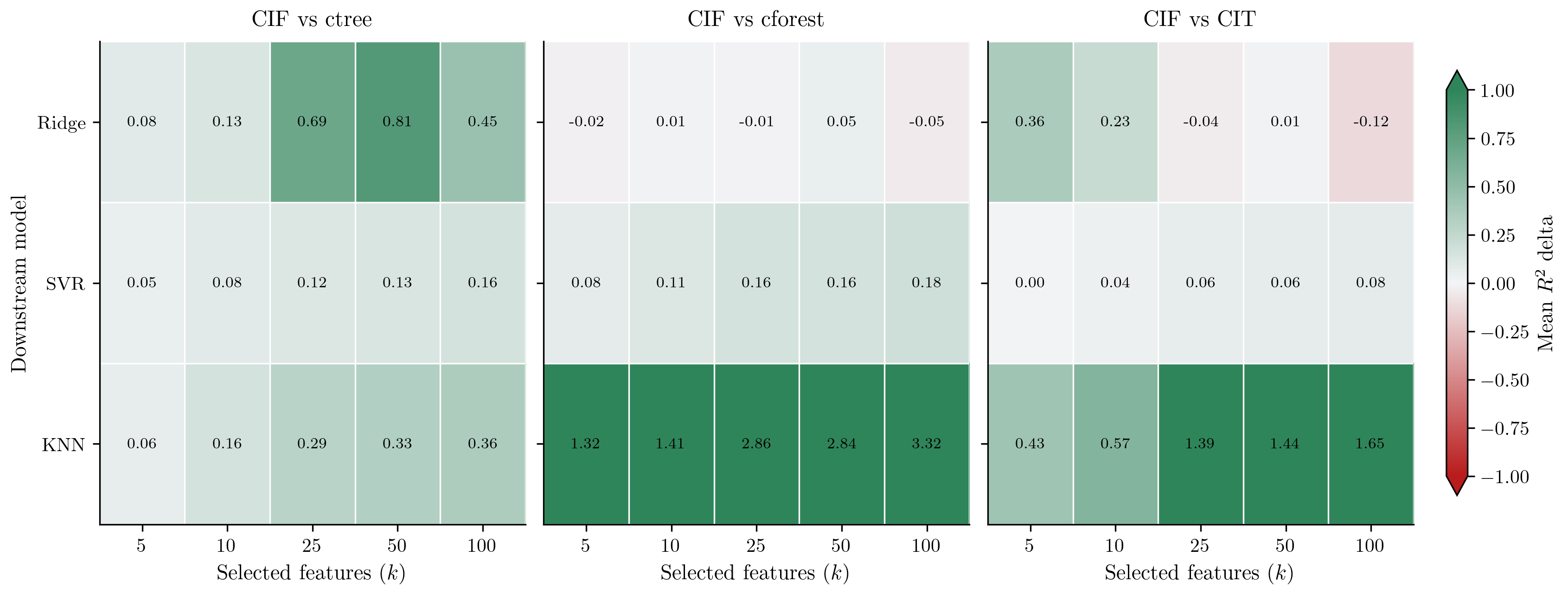}
  \caption{Regression CIF margins against ctree, cforest, and CIT by
  downstream learner and value of $k$. Each plotted value is the mean
  $R^2$ difference, CIF minus comparator, for one downstream
  learner and one value of $k$; positive values favor CIF. Dataset counts by
  $k$ are 8, 8, 7, 7, and 6 for each comparison. Large positive plotted values
  can reflect negative mean $R^2$ for the comparator. Numeric labels show the
  unclipped values; the color scale is clipped at $\pm1$ for readability.}
  \label{fig:regression-benchmark-pairwise-sensitivity}
\end{figure}

For the seed checks, each method is ranked separately within each fixed seed.
Keeping only datasets with results for every method, downstream learner,
standard top-$k$ value, and seed leaves 12 classification datasets and 6
regression datasets complete across all five seeds. These tables summarize seed
variation by rank without intervals or tests. In these tables, a seed column is
the method's ordinal position after sorting that seed's averaged dataset-level
ranks, while mean rank is the averaged dataset-level rank value itself. Rows
follow the main benchmark order.

\begin{table}[H]
  \centering
  \footnotesize
  \setlength{\tabcolsep}{4pt}
  \renewcommand{\arraystretch}{1.04}
  \begin{tabular}{@{}l
    S[table-format=2.0]
    S[table-format=2.0]
    S[table-format=2.0]
    S[table-format=2.0]
    S[table-format=2.0]
    S[table-format=2.1]
    S[table-format=2.1]@{}}
    \toprule
    \tablehead{Method} & {Seed 1} & {Seed 2} & {Seed 3} & {Seed 4} & {Seed 5} & {Mean position} & {Mean rank} \\
    \midrule
    LightGBM & 2 & 1 & 2 & 1 & 1 & 1.4 & 4.7 \\
    XGBoost & 1 & 3 & 1 & 2 & 2 & 1.8 & 5.0 \\
    CatBoost & 3 & 2 & 5 & 3 & 4 & 3.4 & 5.7 \\
    CIF & 5 & 5 & 4 & 4 & 7 & 5.0 & 6.1 \\
    RF & 4 & 4 & 3 & 6 & 3 & 4.0 & 6.0 \\
    RF-RFE & 7 & 7 & 7 & 5 & 6 & 6.4 & 7.2 \\
    ExtraTrees & 6 & 6 & 6 & 7 & 5 & 6.0 & 6.5 \\
    DT & 8 & 9 & 9 & 8 & 8 & 8.4 & 8.5 \\
    RT & 11 & 10 & 12 & 12 & 11 & 11.2 & 10.0 \\
    CIT & 10 & 11 & 10 & 10 & 10 & 10.2 & 9.7 \\
    Boruta & 9 & 8 & 8 & 9 & 9 & 8.6 & 8.4 \\
    MC filter & 12 & 12 & 11 & 11 & 12 & 11.6 & 10.4 \\
    PI & 16 & 16 & 16 & 16 & 16 & 16.0 & 13.1 \\
    cforest & 15 & 15 & 15 & 14 & 15 & 14.8 & 12.6 \\
    RDC filter & 13 & 13 & 13 & 13 & 13 & 13.0 & 11.5 \\
    ctree & 14 & 14 & 14 & 15 & 14 & 14.2 & 12.5 \\
    CPI & 17 & 17 & 17 & 17 & 17 & 17.0 & 15.3 \\
    \bottomrule
  \end{tabular}
  \caption{Classification seed sensitivity on the 12 datasets complete across
  all five seeds. Each seed column is the method's ordinal position after
  sorting methods by their averaged dataset-level ranks for that seed. Mean
  position averages those five ordinal positions; mean rank reports the averaged
  dataset-level rank value itself. Lower values are better.}
  \label{tab:classification-seed-sensitivity}
\end{table}

\begin{table}[H]
  \centering
  \footnotesize
  \setlength{\tabcolsep}{4pt}
  \renewcommand{\arraystretch}{1.04}
  \begin{tabular}{@{}l
    S[table-format=2.0]
    S[table-format=2.0]
    S[table-format=2.0]
    S[table-format=2.0]
    S[table-format=2.0]
    S[table-format=2.1]
    S[table-format=2.1]@{}}
    \toprule
    \tablehead{Method} & {Seed 1} & {Seed 2} & {Seed 3} & {Seed 4} & {Seed 5} & {Mean position} & {Mean rank} \\
    \midrule
    ExtraTrees & 2 & 1 & 1 & 2 & 2 & 1.6 & 6.8 \\
    CatBoost & 1 & 6 & 3 & 1 & 1 & 2.4 & 6.9 \\
    CIF & 5 & 5 & 5 & 4 & 6 & 5.0 & 7.9 \\
    RF & 4 & 3 & 4 & 6 & 5 & 4.4 & 8.0 \\
    RT & 3 & 2 & 2 & 5 & 4 & 3.2 & 7.6 \\
    DT & 11 & 4 & 7 & 3 & 3 & 5.6 & 8.1 \\
    LightGBM & 7 & 8 & 8 & 8 & 7 & 7.6 & 8.7 \\
    RDC filter & 12 & 11 & 9 & 12 & 11 & 11.0 & 9.8 \\
    CIT & 8 & 9 & 12 & 10 & 13 & 10.4 & 9.4 \\
    PC filter & 13 & 13 & 13 & 13 & 12 & 12.8 & 10.5 \\
    DC filter & 15 & 15 & 14 & 15 & 14 & 14.6 & 11.1 \\
    XGBoost & 6 & 7 & 10 & 7 & 8 & 7.6 & 8.8 \\
    Boruta & 10 & 12 & 11 & 11 & 9 & 10.6 & 9.7 \\
    PI & 9 & 10 & 6 & 9 & 10 & 8.8 & 9.1 \\
    RF-RFE & 14 & 14 & 15 & 14 & 15 & 14.4 & 10.9 \\
    cforest & 16 & 17 & 16 & 16 & 17 & 16.4 & 12.3 \\
    ctree & 18 & 18 & 18 & 17 & 18 & 17.8 & 13.1 \\
    CPI & 17 & 16 & 17 & 18 & 16 & 16.8 & 12.3 \\
    \bottomrule
  \end{tabular}
  \caption{Regression seed sensitivity on the 6 datasets complete across all
  five seeds. Each seed column is the method's ordinal position after sorting
  methods by their averaged dataset-level ranks for that seed. Mean position
  averages those five ordinal positions; mean rank reports the averaged
  dataset-level rank value itself. Lower values are better.}
  \label{tab:regression-seed-sensitivity}
\end{table}

\section{Feature use in sampled forests}
\label{app:candidate-availability}

The extended $p$-sweep expands the main-text feature-sampling
analysis, separate from the benchmark-stability summaries in
\cref{app:benchmark-stability}. \Cref{fig:candidate-dimension-curves,fig:regression-candidate-dimension-curves}
show feature-sampling results for $p \in \{100,500,1{,}000\}$.
Each forest uses 1,000 trees, with classification and regression shown
separately. For each method and feature dimension $p$, each point is the mean over
$n_{\mathrm{informative}} \in \{1,2,5,10\}$, and the capped bar spans the
minimum and maximum over those four sparse settings.

In classification, evaluating all nonconstant features at each split gives the
largest informative-feature gains in sparse high-$p$ settings. CIF-all keeps
feature muting and the other tree settings unchanged while removing feature
subsampling. With two informative features and $p{=}100$, CIF uses informative
features in 40.0\% of splits, while CIF-all does so in 100.0\%.
The corresponding values are 12.2\% versus 90.2\% at $p{=}500$ and 9.0\% versus
100.0\% at $p{=}1{,}000$. CIF-all also uses 70, 264, and 439 fewer distinct
uninformative features than CIF. With five informative features, the CIF-all
minus CIF differences are 25.1, 56.4, and 54.3 percentage points
at $p{=}100$, $p{=}500$, and $p{=}1{,}000$.

Regression shows a similar informative-feature use pattern in the one-
and two-informative-feature settings, and CIF-all also uses more distinct
uninformative features. With two informative features and $p{=}100$, CIF uses
informative features in 39.9\% of splits, while CIF-all does so in 58.8\%.
The corresponding values are 19.7\% versus 46.7\% at $p{=}500$ and 11.6\%
versus 35.9\% at $p{=}1{,}000$. With one informative feature, the CIF-all minus
CIF differences are 29.8, 18.9, and 42.1 percentage points.

\begin{figure}[H]
  \centering
  \includegraphics[width=0.88\linewidth]{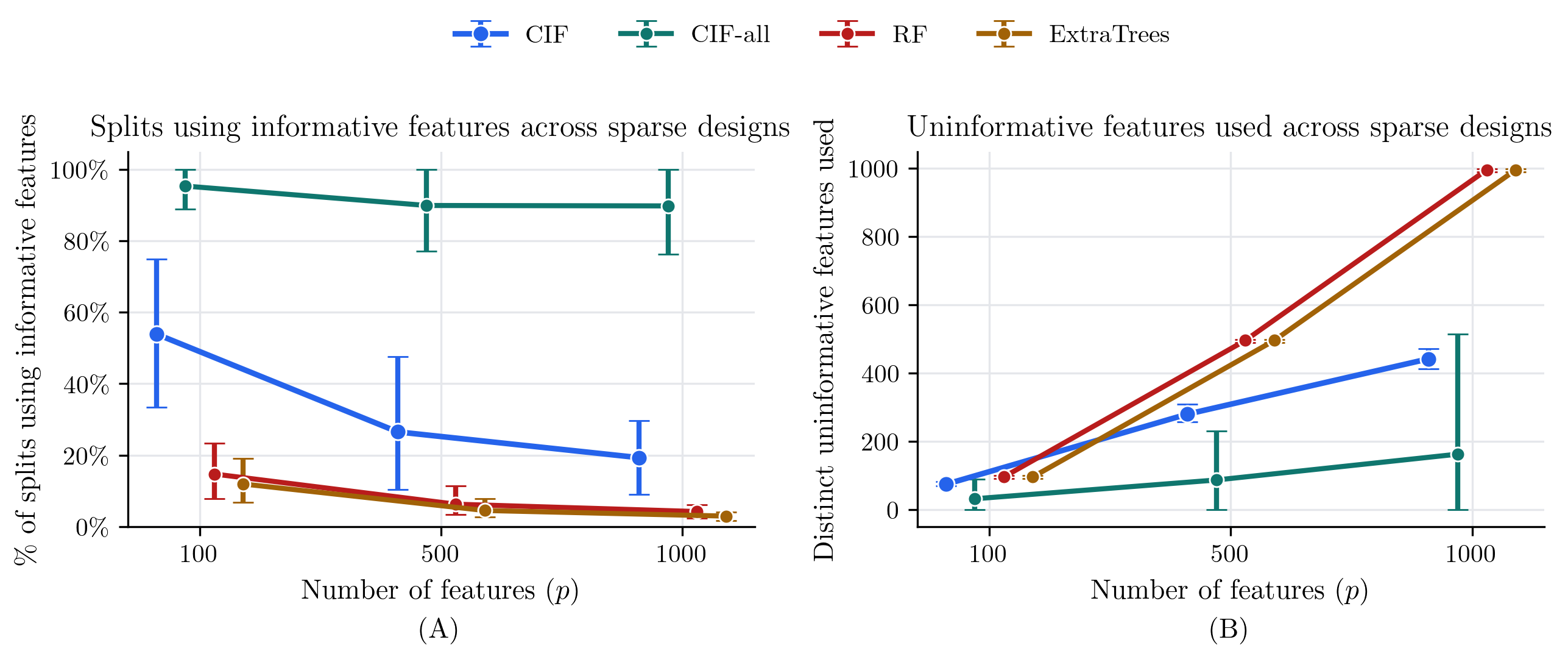}
  \caption{Classification forest feature sampling by feature dimension
  $p$ with 1,000 trees per fit. Points are means over
  $n_{\mathrm{informative}} \in \{1,2,5,10\}$; capped bars span the
  minimum-to-maximum values across those four sparse designs.
  (A) shows \% of splits using informative features; (B) shows
  the number of distinct uninformative features used.}
  \label{fig:candidate-dimension-curves}
\end{figure}

\begin{figure}[H]
  \centering
  \includegraphics[width=0.88\linewidth]{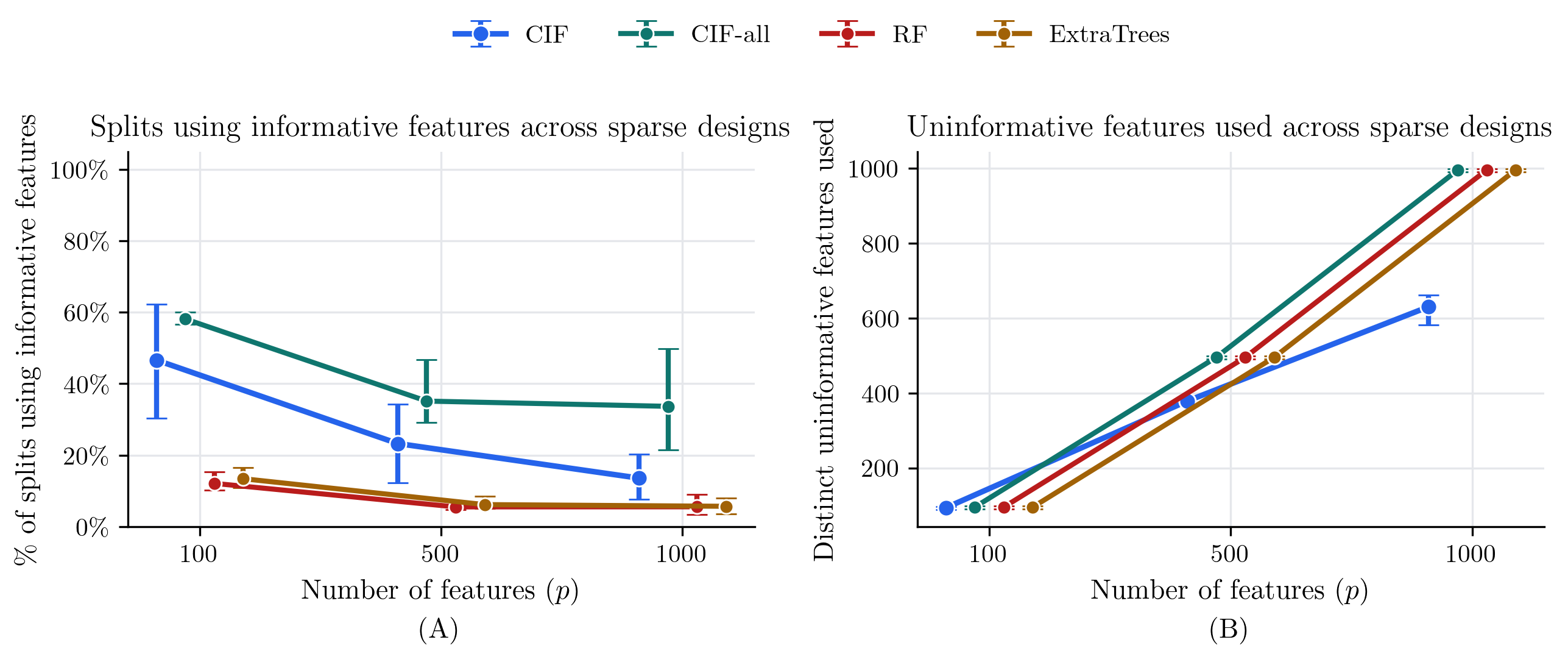}
  \caption{Regression forest feature sampling by feature dimension
  $p$ with 1,000 trees per fit. Points are means over
  $n_{\mathrm{informative}} \in \{1,2,5,10\}$; capped bars span the
  minimum-to-maximum values across those four sparse designs.
  (A) shows \% of splits using informative features; (B) shows
  the number of distinct uninformative features used.}
  \label{fig:regression-candidate-dimension-curves}
\end{figure}

\end{appendices}

\end{document}